\documentclass{article}

    \usepackage[final]{neurips_2020}

\usepackage[utf8]{inputenc} 
\usepackage[T1]{fontenc}    
\usepackage{hyperref}       
\usepackage{url}            
\usepackage{booktabs}       
\usepackage{amsfonts}       
\usepackage{nicefrac}       
\usepackage{microtype}      

\usepackage{macros}
\usepackage{cleveref}
\usepackage{longtable}
\usepackage{algorithm}
\usepackage{algorithmic}
\usepackage{enumitem}
\usepackage{amssymb}
\usepackage{mathalfa}
\usepackage{thmtools}
\usepackage{thm-restate}
\usepackage{wrapfig}
\usepackage[font=footnotesize]{caption}
\DeclareMathAlphabet{\mathpzc}{OT1}{pzc}{m}{it}

\def\Alg{\ensuremath{\textsc{Francis}}}
\def\Lphi{L_\phi}

\def\DeltaRdef{r_t(s_{ti},a_{ti}) - \phi_{ti}^\top\theta_t^r}

\def\IBEE{\mathcal I}

\def\Radius{\mathcal R}

\def\mathringdeffullBe{\max_{(s,a)} \big| [\phi_t(s,a)^\top\mathring \theta_t (Q_{t+1}) - \T^P_t(Q_{t+1})(s,a)]\big| \leq \IBEE(\mathcal Q_{t},\mathcal Q_{t+1})}
\newcommand{\Mdef}[3]{\max_{\pi,\eta \in \R^{d_{#3}} : \| \eta \|_{#2} \leq \sqrt{#1} }\overline \phi_{\pi,{#3}}^\top \eta}

\newcommand{\MagicValue}[1]{8d_p\ln\frac{2d_p}{\delta''}}
\def\sqrtbetadef{\sqrt{2} \times 2\sqrt{\frac{d_t}{2}\ln\(1+\Lphi^2k/d_t\)  + d_{t+1}\ln(1+4\Radius_{t+1}/(2\Lphi\sqrt{k})) + \ln\(\frac{1}{\delta'} \)} + 2}

\def\etatdef{ \Vhat_{t+1}(s^{+}_{t+1,i}) -  \E_{s'\sim p(s_{ti},\pi_{ti}(s_{ti}))} \Vhat_{t+1}(s')}
\def\etardef{  r_{ti} - r_t(s_{ti},a_{ti})}

\def\sigmamaindef{c_\sigma/(d_p\ln(\frac{d_p}{\delta\epsilon}))}
\def\alphamaindef{c_\alpha H^2 (d_p+d_{p+1}) \ln(\frac{d_p}{\epsilon\delta})}
\def\emaxmaindef{c_e \frac{d^2_p\sigma}{\epsilon^2}}
\def\lsvi{\textsc{Lsvi}}
\def\Bound{\widetilde O\(H^2\sum_{t=1}^H \frac{d^2_t (d_t + d_{t+1})}{\epsilon^2}\)}
\def\Ckdef{\Bigg\{ \Mdef{\sigma}{\Sigma_{pk}}{p} > \epsilon'' > \overline \epsilon \Bigg\}}
\def\Ekdef{\Bigg\{ \E_{x_1 \sim \rho} \Vhat_{1k}(x_1) - \overline\epsilon  \geq \Mdef{\sigma}{\Sigma_{pk}}{p}  \Bigg\}}
\def\zetadef{\overline \phi_{\pi_{k(e,i)},p}^\top \xi_{p,k(e,i)} -  \phi_{p,k(e,i)}^\top \xi_{p,k(e,i)}}
\def\Adef{\sqrt{8\ln(\frac{1}{\delta''})}}
\def\gammadef{\sqrt{\frac{2\sigma_{t}d_t}{\lambda_{min}(\Sigma_{p,k(e,i)})}\ln\frac{2d_t}{\delta''}}}
\def\twonormbound{\sqrt{2\sigma_{t}d_t\ln\frac{2d_t}{\delta''}}}

\def\dumpalgo{
\begin{wrapfigure}[22]{r}{0.57\textwidth}
	\vspace{-0.3in}
\begin{minipage}{0.57\textwidth}
\begin{algorithm}[H]
\algsetup{linenosize=\footnotesize}
\footnotesize
\caption{\emph{Forward Reward Agnostic Navigation with Confidence by Injecting Stochasticity} (\Alg{})}
   \label{main.alg:MainAlg}
\begin{algorithmic}[1]
\STATE \textbf{Inputs}: failure probability $\delta \in [0,1]$, target precision $\epsilon >0$, feature map $\phi$
\STATE Initialize $\Sigma_{t1} = \lambda I, \widehat \theta_t = 0, \forall t \in [H]$, $\mathcal D = \emptyset$; set $c_e,c_\sigma,c_\alpha \in \R$ (see appendix), $\lambda = 1$
\FOR{phase $p = 1,2,\dots,H$}
\STATE $k = 1$, 
set $\sigma = \sigma_{start} \defeq \sigmamaindef$
\WHILE{$\sigma < \alphamaindef$} 
\FOR{$i = 1,2,\dots, \emaxmaindef$}
\STATE $k = k+1$, receive starting state $s_{1} \sim \rho$
\STATE $\xi_p \sim \mathcal N(0,\sigma \Sigma^{-1}_{pk})$; \quad $ \textsc{r}_p(s,a) \defeq \phi_p(s,a)^\top\xi_p$
\STATE $\pi \longleftarrow $\lsvi{}($p, \textsc{r}_p,\mathcal D$)
\STATE Run $\pi$; \; $\mathcal D \leftarrow \mathcal D \cup (s_{pk}, a_{pk},s^+_{p+1,k})$;
\STATE $\phi_{pk} \defeq \phi_{p}(s_{pk},a_{pk})$; $\Sigma_{p,k+1} \leftarrow \Sigma_{pk} + \phi_{pk}\phi_{pk}^\top$ 
\ENDFOR
\STATE  $\sigma \longleftarrow 2\sigma$
\ENDWHILE
\ENDFOR 
\STATE \textbf{return} $\mathcal D$
\normalsize
\end{algorithmic}
\end{algorithm}
\end{minipage}
\end{wrapfigure}
}
\def\dumplsvi{
\begin{algorithm}[H]
\algsetup{linenosize=\footnotesize}
\footnotesize
   \caption{\lsvi($H$,$ \textsc{r}_H$, $\mathcal D;\lambda =1$) - This is for use in \Alg{} with reward signal $\textsc{r}_H$}
   \label{alg:lsviFrancis}
\begin{algorithmic}[1]
\STATE \textbf{Input}: horizon $H$, dataset $\mathcal D$, regularization $\lambda$. \\
\STATE Extract pseudo-reward parameter $\xi_H$ from $\textsc{r}_H$ function
\STATE Set $\widehat \theta_H = \xi_H$
\FOR{timestep $t = H-1,\dots,1$}
\STATE Solve $\widehat \theta_t = \argmin_{\theta} \sum_{k=1}^{n(t)} \big[\phi_{t}(s_{tk},a_{tk})^\top \theta - \max_{a'}\phi_{t+1}(s_{t+1,k}^{+},a')^\top \widehat \theta_{t+1} \big]^2 + \lambda \| \theta\|_2^2$
\ENDFOR
\STATE \textbf{Return $\pi: (s,t) \mapsto \argmax_{a} \phi_t(s,a)^\top \widehat \theta_t$}
\end{algorithmic}
\end{algorithm}

\begin{algorithm}[H]
\algsetup{linenosize=\footnotesize}
\footnotesize
   \caption{\lsvi($H$, $\mathcal D;\lambda =1$) - This is the regular batch algorithm}
   \label{alg:lsviBatch}
\begin{algorithmic}[1]
\STATE \textbf{Input}: horizon $H$, dataset $\mathcal D$, regularization $\lambda$. \\
\STATE \textbf{Set} $\widehat \theta^{R+PV}_{H+1} = 0$.
\FOR{timestep $t = H,H-1,\dots,1$}
\STATE Solve $\widehat \theta^{R+PV}_t = \argmin_{\theta} \sum_{k=1}^{n(t)} \big[\phi_{t}(s_{tk},a_{tk})^\top \theta - r_{tk} - \max_{a'}\phi_{t+1}(s_{t+1,k}^{+},a')^\top \widehat \theta^{R+PV}_{t+1} \big]^2 + \lambda \| \theta\|_2^2$
\ENDFOR
\STATE \textbf{Return $\pi: (s,t) \mapsto \argmax_{a} \phi_t(s,a)^\top \widehat \theta^{R+PV}_t$}
\end{algorithmic}
\end{algorithm}
}

\title{Provably Efficient Reward-Agnostic Navigation \\ with Linear Value Iteration}

\author{%
  Andrea Zanette \\
    Stanford University \\
  \texttt{zanette@stanford.edu} \\
  \And
  Alessandro Lazaric \\
  Facebook Artificial Intelligence Research \\
  \texttt{lazaric@fb.com} \\
  \AND
  Mykel J. Kochenderfer \\
  Stanford University \\
  \texttt{mykel@stanford.edu} \\
  \And
  Emma Brunskill \\
  Stanford University \\
  \texttt{ebrun@cs.stanford.edu} \\
}

\begin{document}

\maketitle

\begin{abstract}
There has been growing progress on theoretical analyses for provably efficient learning in MDPs with linear function approximation, but much of the existing work has made strong assumptions to enable exploration by conventional exploration frameworks. Typically these assumptions are stronger than what is needed to find good solutions in the batch setting. In this work, we show how under a more standard notion of low inherent Bellman error, typically employed in least-square value iteration-style algorithms, we can provide strong PAC guarantees on learning a near optimal value function provided that the linear space is sufficiently ``explorable''. 
We present a computationally tractable algorithm for the reward-free setting and show how it can be used to learn a near optimal policy for any (linear)  reward function, which is revealed only once learning has completed.  If this reward function is also estimated from the samples gathered during pure exploration, our results also provide same-order PAC guarantees on the performance of the resulting policy for this setting. 
\end{abstract}

\section{Introduction}
Reinforcement learning (RL) aims to solve complex multi-step decision problems with stochastic outcomes framed as a Markov decision process (MDP).
RL algorithms often need to explore large state and action spaces where function approximations become necessity. In this work, we focus on exploration with linear predictors for the action value function, which can be quite expressive \citep{sutton2018reinforcement}.

\paragraph{Existing guarantees for linear value functions}
Exploration has been widely studied in the tabular setting \citep{Azar17,zanette2019tighter,efroni2019tight,jin2018q,dann2019policy}, but obtaining formal guarantees for exploration with function approximation appears to be a challenge even in the linear case.
The minimal necessary and sufficient conditions to reliably learn a linear predictor are not fully understood \emph{even with access to a generative model} \citep{du2019good}. We know that when the best policy is unique and the predictor is sufficiently accurate it can be identified \citep{du2019provably,du2020agnostic}, but in general we are  interested in finding only near-optimal policies using potentially misspecified approximators.

To achieve this goal, several ideas from tabular exploration and linear bandits \citep{lattimore2020bandit} have been combined to obtain provably efficient algorithms in low-rank MDPs  \citep{yang2020reinforcement,zanette2020frequentist,jin2020provably} and their extension \citep{wang2019optimism,wang2020provably}. We shall identify the core assumption of the above works  as \emph{optimistic closure}: all these settings assume the \emph{Bellman operator maps any value function of the learner to a low-dimensional space $\mathcal Q$} that the learner knows. When this property holds, we can add exploration bonuses because \emph{by assumption the Bellman operator maps the agent's optimistically modified value function back to $\mathcal Q$}, which the algorithm can represent and use to propagate the optimism and drive the exploration. However, the optimistic closure is put as an assumption to enable \emph{exploration} using traditional methods, but is stronger that what is typically required in the batch setting.  

\paragraph{Towards batch assumptions}
This work is motivated by the desire to have exploration algorithms that we can deploy under more mainstream assumptions, ideally when we can apply well-known \emph{batch} procedures like least square policy iteration (\textsc{Lspi}) \citep{lagoudakis2003least}, and least square value iteration (\lsvi{}) \citep{munos2005error}.

\textsc{Lspi} has convergence guarantees when the action value function of \emph{all} policies can be approximated with a linear architecture \citep{lazaric2012finite}, i.e., $Q^{\pi}$ is linear for all $\pi$; in this setting, \citet{lattimore2020learning} recently use a design-of-experiments procedure from the bandit literature to obtain a provably efficient algorithm for finding a near optimal policy, but they need access to a generative model. \lsvi{}, another popular batch algorithm, requires low inherent Bellman error \citep{munos2008finite,chen2019information}.
In this setting, \citet{zanette2020learning} present a near-optimal (with respect to noise and misspecification) regret-minimizing algorithm that operates online, but a computationally tractable implementation is not known. It is worth noting that both settings are more general than linear MDPs \citep{zanette2020learning}.

A separate line of research is investigating settings with low Bellman rank \citep{jiang17contextual} which was found to be a suitable measure of the learnability of many complex reinforcement learning problems. The notion of Bellman rank extends well beyond the linear setting.

The lack of computational tractability in the setting of \citet{zanette2020learning} and in the setting with low Bellman rank \citep{jiang17contextual} and of a proper online algorithm in \citep{lattimore2020learning} highlight the hardness of these very general settings which do not posit additional assumptions on the linear value function class $\mathcal Q$ beyond what is required in the batch setting.

\paragraph{Reward-free exploration}
We tackle the problem of designing an exploration algorithm using batch assumptions by adopting a pure exploration perspective: our algorithm can return a near optimal policy for any linear reward function that is revealed after an initial learning phase. It is therefore a probably approximately correct (PAC) algorithm.
Reward-free exploration has been investigated in the tabular setting with an end-to-end algorithm \citep{jin2020rewardfree}.  \citet{hazan2018provably} design an algorithm for a more general setting through oracles that also recovers guarantees in the tabular domains. Others \citep{du2019provablyefficient,misra2019kinematic} also adopt the pure exploration perspective assuming a small but unobservable state space. More recently, reward free exploration has gained attention in the tabular setting \cite{kaufmann2020adaptive,tarbouriechreward,menard2020fast} as well as the context of function approximation \cite{wainwright2019high,agarwal2020pc}.

\paragraph{Contribution}
This works makes two contributions. It presents a  statistically and computationally efficient online PAC algorithm to learn a near-optimal policy 1) for the setting with low inherent Bellman error \citep{munos2008finite} and 2) for reward-free exploration in the same setting.

From a technical standpoint, 1) implies we cannot use traditional exploration methodologies and 2) implies we cannot learn the full dynamics, which would require estimating all state-action-state transition models.
Both goals are accomplished by driving exploration by approximating G-optimal experimental design \citep{lattimore2020bandit} in online reinforcement learning through randomization. 
Our algorithm returns a dataset of well chosen state-action-transition triplets, such that invoking the \lsvi{} algorithm on that dataset (with a chosen reward function) returns a near optimal policy on the MDP with that reward function.

\section{Preliminaries and Intuition}
We consider an undiscounted $H$-horizon MDP~\citep{puterman1994markov} $M = (\mathcal{S}, \mathcal{A},p, r, H)$ defined by a possibly infinite state space $\mathcal{S}$ and action space $\mathcal{A}$. For every $t \in [H] = \{1, \ldots, H\}$ and state-action pair $(s,a)$, we have a reward function $r_t(s,a)$ and a transition kernel $p_t(\cdot \mid s,a)$ over the next state.
A policy $\pi$ maps a $(s,a,t)$ triplet to an action and defines a reward-dependent action value function $Q^{\pi}_{t}(s,a) = r_{t}(s,a) + \mathbb{E} \left[ \sum_{l = t+1}^{H} r_{l}(s_{l}, \pi_{l}(s_l))\mid s,a \right]$
and a value function $V^{\pi}_t(s) = Q^{\pi}_t(s, \pi_t(s))$. For a given reward function there exists an optimal policy $\pi^\star$ whose value and action-value functions on that reward function are defined as $\Vstar_t(s ) =  \sup_{\pi} V^{\pi}_t(s)$ and $\Qstar_t(s,a) = \sup_{\pi} Q^{\pi}_t(s,a)$. We indicate with $\rho$ the starting distribution. 
The Bellman operator $\T_t$ applied to the action value function $Q_{t+1}$ is defined as $\T_t(Q_{t+1})(s,a) = r_t(s,a) + \E_{s' \sim p_t(s,a)} \max_{a'}Q_{t+1}(s',a')$. For a symmetric positive definite matrix $\Sigma$ and a vector $x$ we define $\| x \|_{\Sigma^{-1}} = \sqrt{x^\top\Sigma^{-1}x}$. The $O(\cdot)$ notation hides constant values and the $\widetilde O(\cdot)$ notation hides constants and $\ln(dH\frac{1}{\epsilon}\frac{1}{\delta})$, where $d$ is the feature dimensionality described next.

\paragraph{Linear Approximators}
For the rest of the paper we restrict our attention to linear functional spaces for the action value function, i.e.,  where $ Q_t(s,a) \approx \phi_t(s,a)^\top\theta$ for a known feature extractor $\phi_t(s,a)$ and a parameter $\theta$ in a certain set $\mathcal B_t$, which we assume to be the Euclidean ball with unit radius $\mathcal B_t = \{\theta \in \R^{d_t} \mid \| \theta \|_2\leq  1 \}$.
This defines the value functional spaces as
\begin{align*}
	\mathcal Q_{t} \defeq \{ Q_t \mid Q_t(s,a) = \phi_t(s,a)^\top\theta, \; \theta\in\mathcal B_t \}, \quad
	\mathcal V_{t} \defeq \{ V_t \mid V_t(s) = \max_{a} \phi_t(s,a)^\top\theta, \; \theta\in\mathcal B_t \}.
\end{align*}

\paragraph{Inherent Bellman error}
The inherent Bellman error condition is typically employed in the analysis of \lsvi{} \citep{munos2008finite,chen2019information}. It measures the closure of the prescribed functional space $\mathcal Q$ with respect to the Bellman operator $\T$, i.e, the distance of $\T Q$ from $\mathcal Q$ \emph{provided that $Q \in \mathcal Q$}. In other words, low inherent Bellman error ensures that if we start with an action value function in $\mathcal Q$ then we approximately remain in the space after performance of the Bellman update. For finite horizon MDP we can define the inherent Bellman error as:
\begin{align}
\label{main.eqn:IBE_old}
\max_{\substack{Q_{t+1}\in \mathcal Q_{t+1}}} \min_{Q_t \in \mathcal Q_{t}} \max_{(s,a)} |[Q_t - \T_t(Q_{t+1})](s,a)|.
\end{align}
When linear function approximations are used and the inherent Bellman error is zero, we are in a setting of low Bellman rank \citep{jiang17contextual}, where the Bellman rank is the feature dimensionality. This condition is more general than the low rank MDP setting or optimistic closure \citep{yang2020reinforcement,jin2020provably,zanette2020frequentist,wang2019optimism}; for a discussion of this see \citep{zanette2020learning}.

\paragraph{Model-free reward-free learning}
In the absence of reward signal, how should $\mathcal Q_t$ look like? Define the reward-free Bellman operator $\T^P_t(Q_{t+1})(s,a) = \E_{s' \sim p_t(s,a)} \max_{a'}Q_{t+1}(s',a')$. 
It is essentially equivalent to measure the Bellman error either on the full Bellman operator $\T_t$ or directly on the dynamics $\T_t^P$ when the reward function is linear (see proposition 2 of \citet{zanette2020learning}). We therefore define the inherent Bellman error directly in the transition operator $\T^P$:
\begin{definition}[Inherent Bellman Error]
\label{main.def:InherentBellmanError}
\begin{align}
\IBEE(\mathcal Q_{t},\mathcal Q_t) \defeq \max_{\substack{Q_{t+1}\in \mathcal Q_{t+1}}} \min_{Q_t \in \mathcal Q_{t}} \max_{(s,a)} |Q_t - \T^{P}_t(Q_{t+1})](s,a)|.
\end{align}
\end{definition}

\paragraph{Approximating G-optimal design}

G-optimal design is a procedure \citep{kiefer1960equivalence} that identifies an appropriate sequence of features $\phi_1,\dots\phi_n$ to probe to form the design matrix $\Sigma = \sum_{i=1}^{n} \phi_i\phi_i^\top$ in order to uniformly reduce the maximum ``uncertainty'' over all the features as measured by $\max_{\phi}\| \phi\|_{\Sigma^{-1}} $, see \cref{sec:DoE}.
This principle has recently been applied to RL with a generative model \citep{lattimore2020learning} to find a near optimal policy.

However, the basic idea has the following drawbacks in RL: 1) it requires access to a generative model; 2) it is prohibitively expensive as it needs to examine all the features across the full state-action space before identifying what features to probe.
This work addresses these 2 drawbacks in reinforcement learning by doing two successive approximations to G-optimal design. The first approximation would be compute and follow the policy $\pi$ (different in every rollout) that leads to an expected feature $\overline \phi_{\pi}$ in the most uncertain direction\footnote{This is an approximation to $G$-optimal design, because $\pi$ here is the policy that leads to the most uncertain direction $\overline \phi_{\pi}$ rather than to the direction that reduces the uncertainty the most.} (i.e., the direction where we have the least amount of data). This solves problem 1 and 3 above, but unfortunately it turns out that computing such $\pi$ is computationally infeasible. Thus we relax this program by finding a policy that in most of the episodes makes at least some progress in the most uncertain direction, thereby addressing point 2 above. This is achieved through randomization; the connection is briefly outlined in \cref{main.sec:ConnectionwithG-optimaldesign}.

\section{Algorithm}
Moving from the high-level intuition to the actual algorithm requires some justification, which is left to \cref{main.sec:TechicalAnalysis}. Here instead we give few remarks about  \cref{main.alg:MainAlg}: first, the algorithm proceeds in phases $p=1,2,\dots$ and in each phase it focuses on learning the corresponding timestep (e.g., in phase $2$ it learns the dynamics at timestep $2$). \dumpalgo
Proceeding forward in time is important because \emph{to explore at timestep $p$ the algorithm needs to know how to navigate through prior timesteps}. Second, we found that random sampling a reward signal in the exploratory timestep from the inverse covariance matrix $\xi_p \sim \mathcal N(0,\sigma\Sigma^{-1}_{pk})$ is an elegant and effective way to \emph{approximate design of experiment} (see \cref{main.sec:ConnectionwithG-optimaldesign}), although this is not the only possible choice.
Variations of this basic protocol are broadly known in the literature as Thompson sampling \citep{osband2016deep,Agrawal2017,russo2019worst,gopalan2015thompson,ouyang2017learning} and from an algorithmic standpoint our procedure could be interpreted as a modification of the popular \textsc{Rlsvi} algorithm \citep{osband2016generalization} to tackle the reward-free exploration problem. 

The algorithm returns a dataset $\mathcal D$ of well chosen state-action-transitions approximating a G-optimal design in the online setting; the dataset can be augmented with the chosen reward function and used in \lsvi{} (detailed in \cref{sec:lsvi}) to find a near-optimal policy on the MDP with that reward function. The call \lsvi{}($p, \textsc{r}_p,\mathcal D$) invokes the \lsvi{} algorithm on a $p$ horizon MDP on the batch data $\mathcal D$ with reward function $\textsc{r}_p$ at timestep $p$.

\section{Main Result}
\label{main.sec:MainResult}
Before presenting the main result is useful to define the average feature $\overline \phi_{\pi,t} = \E_{x_t \sim \pi}\phi_{t}(x_t,\pi_t(x_t))$ encountered at timestep $t$ upon following a certain policy $\pi$. In addition, we need a way to measure how ``explorable'' the space is, i.e., how easy it is to collect information in a given direction of the feature space using an appropriate policy. The explorability coefficient $\nu$ measures how much we can align the expected feature $\overline \phi_{\pi,t}$   with the most challenging direction $\theta$ to explore even if we use the \emph{best} policy $\pi$ for the task (i.e., the policy that maximizes this alignment). It measures how difficult it is to explore the most challenging direction, even if we use the best (and usually unknown) policy to do so. This is similar to a diameter condition in the work of \citet{Jaksch10} in the features space, but different from ergodicity, which ensures that sufficient information can be collected by \emph{any} policy. It is similar to the reachability parameter of \citet{du2019provablyefficient} and \citet{misra2019kinematic}, but our condition concerns the features rather than the state space and is unavoidable in certain settings (see discussion after the main theorem).
\begin{definition}[Explorability]
	\label{main.def:Explorability}
$
		 \nu_t \defeq \min_{\|\theta\|_2=1}\max_{\pi} | \overline \phi_{\pi,t}^\top \theta |; \quad \quad \nu_{min} = \min_{t \in [H]} \nu_t.
$
\end{definition}
\black

\begin{restatable}[]{thm}{mainresult}
\label{main.thm:MainResult}
Assume $ \| \phi_t(s,a) \|_2 \leq 1$ and set $\epsilon$ to satisfy $\epsilon \geq \widetilde O(d_tH\IBEE(\mathcal Q_{t},\mathcal Q_{t+1}))$ and $\epsilon \leq \widetilde O(\nu_{min}/\sqrt{d_t})$ for all $t \in [H]$.  \Alg{} terminates after
	$\Bound$ episodes. 
	
	Fix a reward function $r_t(\cdot,\cdot)$ such that each state-action-successor state  $(s_{tk},a_{tk},s^+_{t+1,k})$  triplet in $\mathcal D$ (where $t \in [H]$ and $k$ is the episode index in phase $t$) is augmented with a reward $r_{tk} = r_t(s_{tk},a_{tk})$.  
	If the reward function $r_{t}(\cdot,\cdot)$ satisfies for some parameters $ \theta_1^r \in \R^{d_1}, \dots, \theta_H^r \in \R^{d_H}$
	\begin{align*}
	\forall (s,a,t) \quad \| \theta^r_t \|_2 \leq\frac{1}{H}, \quad r_t(s,a) = \phi_{t}(s,a)^\top \theta^r_t
	\end{align*}
		then with probability at least $1-\delta$ the policy $\pi$ returned by \lsvi{} using  the augmented dataset $\mathcal D$ satisfies (on the MDP with $r_{t}(\cdot,\cdot)$ as reward function) 
	\begin{align}
		\E_{x_1 \sim \rho} (V_1^\star - V^\pi_1)(x_1) \leq \epsilon.
	\end{align}
\end{restatable}

The full statement is reported in appendix \cref{sec:SolutionReconstruction}.
The reward function $r_t(\cdot,\cdot)$ could even be adversarially chosen after the algorithm has terminated. If the reward function is estimated from data then the theorem immediately gives same-order guarantees as a corollary. The dynamics error $O(d_tH\IBEE(\mathcal Q_{t},\mathcal Q_{t+1}))$ is contained in $\epsilon$.

The setting allows us to model MDPs where where $r_t \in [0,\frac{1}{H}]$ and $\Vstar_t \in [0,1]$. When applied to MDPs with rewards in $[0,1]$ (and value functions in $[0,H]$), the input and output should be rescaled and the number of episodes to $\epsilon$ accuracy should be multiplied by $H^2$. 

The significance of the result lies in the fact that this is the first statistically and computationally\footnote{\Alg{} requires only polynomial calls to \lsvi{} and samples from a multivariate normal,  see \cref{sec:ComputationalComplexity}.} efficient PAC algorithm for the setting of low inherent Bellman error; this is special case of the setting with low Bellman rank (the Bellman rank being the dimensionality of the features). 
In addition, this work provides one of the first end-to-end algorithms for provably efficient reward-free exploration with linear function approximation.

\begin{table}
\footnotesize
\begin{tabular}{@{}p{3.2cm} p{1.0cm} p{1.2cm} p{2.2cm} p{1.4cm} p{2.4cm}@{}}
 \toprule
  & Online? & Reward-agnostic? & Need optimistic closure?&$\#$ episodes & $\#$ computations  \\
 \midrule
 This work  & Yes & Yes  & No  &$\frac{d^3H^5}{\epsilon^2}$ &  poly$(d,H,1/\epsilon^2)$  \\
 G-optimal design + \lsvi{} & No & Yes & No & $\frac{d^2H^5}{\epsilon^2}$ &  $\Omega(SA)$ \\
  \citep{zanette2020learning} & Yes & No & No   & $\frac{d^2H^4}{\epsilon^2}$   & exponential  \\
 \citep{jin2020provably} & Yes & No & Yes &  $\frac{d^3H^4}{\epsilon^2}$   & poly$(d,H,1/\epsilon^2)$  \\
 \citep{jiang17contextual} & Yes & No & No  & $  \frac{d^2H^5}{\epsilon^2}|\mathcal A| $   & intractable  \\
 \citep{jin2020rewardfree} & Yes & Yes & (tabular) & $\frac{H^5S^2A}{\epsilon^2} $ &  poly$(S,A,H,1/\epsilon^2)$ \\
  \citep{wang2020reward} & Yes & Yes & Yes & $\frac{d^3H^6}{\epsilon^2} $ &  poly$(S,A,H,1/\epsilon^2)$ \\
 \bottomrule
\end{tabular}
\vspace{0.1cm}
\caption{\footnotesize We consider the number of episodes to learn an $\epsilon$-optimal policy. We assume $r \in [0,1]$ and $Q^\pi \in [0,H]$, and rescale the results to hold in this setting. We neglect misspecification for all works. The column ``optimistic closure'' refers to the assumption that the Bellman operator projects \emph{any} value function into a prescribed space (notably, low-rank MDPs of \citep{jin2020provably}). For our work we assume $\epsilon = \Omega(\nu_{min}/\sqrt{d})$.  We recall that if an algorithm has regret $A\sqrt{K}$, with $K$ the number of episodes then we can extract a PAC algorithm to return an $\epsilon$-optimal policy in $\frac{A^2}{\epsilon^2}$ episodes. We evaluate \citep{jiang17contextual} in our setting where the Bellman rank is $d$ (the result has an explicit dependence on the number of actions, though this could be improved in the linear setting). $G$-optimal design is from the paper \citep{lattimore2020learning} which operates in infinite-horizon and assuming linearity of $Q^\pi$ for all $\pi$, so the same idea of G-optimal design was applied to our setting to derive the result and we report the number of required samples (as opposed to the number of episodes), see \cref{sec:DoE}. For \citep{jin2020rewardfree} we ignore the $\frac{H^7S^4A}{\epsilon}$ lower order term}
\label{main.table:MainTable}
\end{table}
\normalsize

In \cref{main.table:MainTable} we describe our relation with few relevant papers in the field. The purpose of the comparison is not to list the pros and cons of each work with respect to one another, as these works all operate under different assumptions, but rather to highlight what is achievable in different settings.

\paragraph{Is small Bellman error needed?}
As of writing, the minimal conditions that enable provably efficient learning with function approximation are still unknown \citep{du2019good}. In this work we focus on small Bellman error which is a condition typically used for \emph{batch} analysis of \lsvi{} \citep{munos2005error,munos2008finite,chen2019information}.
What is really needed for the functioning of \Alg{} is that vanilla \lsvi{} outputs a good solution in the limit of infinite data on different (linear) reward functions: as long as \lsvi{} can return a near-optimal policy for the given reward function given enough data, \Alg{} can proceed with the exploration. This requirement is really minimal, because even if the best dataset $\mathcal D$ is collected through G-optimal design on a generative model (instead of using \Alg{}), \lsvi{} must anyway be able to output a good policy on the prescribed reward function.

\paragraph{Is explorability needed?}
Theorem \ref{main.thm:MainResult} requires 
$\epsilon \leq \widetilde O(\nu_{min}/\sqrt{d_t})$. Unfortunately, a dependence on $\nu_{min}$ turns out to be unavoidable in the \emph{more general} setting we consider in the appendix; 
we discuss this in more detail in \cref{sec:LowerBound}, but here we give some intuition regarding the explorability requirement. 

\Alg{} can operate under two separate set of assumptions, which we call
\emph{implicit} and \emph{explicit} regularity, see \fullref{def:RewardClasses} in appendix and the main result in \cref{thm:MainResult}.

Under \emph{implicit} regularity we do not put assumptions on the norm of reward parameter $\| \theta^r\|_2$, but only a bound on the expected value of the rewards under any policy: $|\mathbb E_{x_t \sim \pi} r_t(x_t,\pi_t(x_t)) | \leq \frac{1}{H}$. This representation allows us to represent \emph{very high rewards} $(\gg 1)$ \emph{in hard-to-reach states}. It basically controls how big the value function can get. This setting is more challenging for an agent to explore \emph{even in the tabular setting} and \emph{even in the case of a single reward function}. If a state is hard to reach, the reward there can be very high, and a policy that tries to go there \emph{can still have high value}. Under this implicit regularity assumption, the explorability parameter would show up for tabular algorithms as well (as minimum visit probability to any state under an appropriate policy). 

By contrast, under
\emph{explicit} regularity (which concerns the result reported in \cref{main.thm:MainResult})
 we do make the classical assumption that bounds the parameter norm $\| \theta^r \|_2 \leq 1/H$. In this case, the lower bound no longer applies, but the proposed algorithm still requires good ``explorability'' to proceed. Removing this assumption is left as future work.

\section{Technical Analysis}
\label{main.sec:TechicalAnalysis}

For the proof sketch we neglect misspecification, i.e., $\mathcal \IBEE(\mathcal Q_{t},\mathcal Q_{t+1}) = 0$. We say that a statement holds with very high probability if the probability that it does not hold is $\ll \delta$.

\subsection{Analysis of \lsvi{}, uncertainty and inductive hypothesis}
\Alg{} repeatedly calls \lsvi{} on different randomized linearly-parameterized reward functions $\textsc{r}_p$ and so we need to understand how the signal propagates. 
Let us begin by defining an uncertainty function in episode $i$ of phase $p$ using the covariance matrix $\Sigma_{pi} = \sum_{j=1}^{i-1} \phi_{pj}\phi_{pj}^\top + I$ on the observed features $\phi_{pj} = \phi_p(s_{pj},a_{pj})$ at episode $j$ of phase $p$:
\begin{definition}[Max Uncertainty]
\label{main.def:Ufnc} $\mathcal U^\star_{pi}(\sigma) \defeq \max_{\pi,\| \theta^{\mathcal U} \|_{\Sigma_{pi}\leq \sqrt{\sigma}}}
 \overline \phi_{\pi,p}^\top\theta^{\mathcal U} \defeq \max_{\pi}\sqrt{\sigma}\| \overline \phi_{\pi,p} \|_{\Sigma^{-1}_{pi}}$.
\end{definition}
Let  $\Sigma_t$ denote the covariance matrix in timestep $t$ once learning in that phase has completed, and likewise denote with $\mathcal U^\star_t(\sigma)$ the final value of the program of \cref{main.def:Ufnc} once learning in phase $t$ has completed (so using $\Sigma_{t}$ in the definition); let $\sqrt{\alpha_t} = \widetilde O(\sqrt{d_t + d_{t+1}})$  and $\textsc{r}_p(s,a) = \phi_p(s,a)^\top\xi_p$.
\begin{lemma}[see \cref{sec:lsvi-final-bound}]
\label{main.lem:lsvi}
Assume $\| \xi_p \|_2 \leq 1$ and $\lambda_{min}(\Sigma_t) = \Omega(H^2\alpha_t)$ for all $t \in [p-1]$. Then with very high probability \lsvi($p,\textsc{r}_p,\mathcal D$) computes a value function $\Vhat$ and a policy $\pi$ s.t.
	\begin{align*}
|\E_{x_1 \sim \rho}\Vhat_{1}(x_1) - \overline \phi_{\pi,p}^\top \xi_p |
	\leq \sum_{t=1}^{p-1} \Big[ \sqrt{\alpha_{t}} \| \overline \phi_{\pi,t} \|_{\Sigma^{-1}_{t}}  \Big] = \sum_{t=1}^{p-1} \mathcal U^\star_{t}(\alpha_t)= \text{Least-Square Error}.
\end{align*}
\end{lemma}
The least-square error in the above display can be interpreted as a planning error to propagate the signal $\xi_p$; it also appears when \lsvi{} uses the batch dataset $\mathcal D$ to find the optimal policy on a given reward function after \Alg{} has terminated, and it is the quantity we target to reduce. Since $\alpha_t$ is constant, we need to shrink $\| \overline \phi_{\pi,p} \|_{\Sigma^{-1}_{p}}$ over any choice of $\pi$ as much as possible by obtaining an appropriate\footnote{G-optimal design does this optimally, but requires choosing the features, which is only possible if one has access to a generative model or in a bandit problem.} feature matrix $\Sigma_{t}$.

A final error across all timesteps of order $\epsilon$ can be achieved when the algorithm adds at most $\epsilon/H$ error at every timestep. Towards this, we define an inductive hypothsis that the algorithm has been successful up to the beginning of phase $p$ in reducing the uncertainty encoded in $\mathcal U^\star_{t}$:
\begin{inductivehypothesis} At the start of phase $p$ we have $\sum_{t=1}^{p-1} \mathcal U^\star_{t}(\alpha_t)  \leq \frac{p-1}{H}\epsilon$.
\end{inductivehypothesis}
The inductive hypothesis \emph{critically ensures} that the reward signal $\xi$ can be accurately propagated backward by \lsvi{}, enabling navigation capabilities of \Alg{} to regions of uncertainty in phase $p$ (this justifies the phased design of \Alg{}).

\subsection{Overestimating the maximum uncertainty through randomization}
Assuming the inductive hypothesis, we want to show how to reduce the uncertainty in timestep $p$.
Similar to how optimistic algorithms overestimate the optimal value function, here $\E_{x_1 \sim \rho}\Vhat_1(x_t) \approx \overline \phi_{\pi,p}^\top \xi_p$ should overestimate the current uncertainty in episode $i$ of phase $p$ encoded in $\mathcal U^\star_{pi}(\alpha_p)$. This is achieved by introducing a randomized reward signal $\xi_{pi} \sim \mathcal N(0,\sigma \Sigma^{-1}_{pi})$ at timestep $p$.
\begin{lemma}[Uncertainty Overestimation, \cref{sec:Derandomization}]
	\label{main.lem:RelaxedProgram}
	If $\xi_p \sim \mathcal N(0,\sigma\Sigma^{-1}_{pi})$, $\mathcal U^\star_{pi}(\sigma) = \Omega(\epsilon)$, $\| \xi_p\|_2 \leq 1$ and the inductive hypothesis holds then \lsvi{} returns with some constant probability $q \in \R$ a policy $\pi$ such that $\overline \phi_{\pi,p}^\top\xi_{pi} \geq \mathcal U^\star_{pi}(\sigma)$.
\end{lemma}

The proof of the above lemma uses \cref{main.lem:lsvi}. The condition $\mathcal U^\star_{pi}(\sigma) = \Omega(\epsilon)$ is needed: if the signal $\xi_{pi}$ or uncertainty $\mathcal U^\star_{pi}(\sigma)$ are too small relative to $\epsilon$ then the least-square error of order $\epsilon$ that occurs in \lsvi{} is too large relative to the signal $\xi_{pi}$, and the signal cannot be propagated backwardly.

The lemma suggests we set $\sigma = \alpha_t$ to ensure $\overline \phi_{\pi,p}^\top\xi_{pi} \geq \mathcal U^\star_{pi}(\alpha_t)$ with fixed probability $q \in \R$. Unfortunately this choice would generate a very large $\| \xi_{pi} \|_2$ which violates the condition $\| \xi_{pi} \|_2 \leq 1$. In particular, the condition $\| \xi_{pi} \|_2 \leq 1$ determines how big $\sigma$ can be.

\begin{lemma}[see \cref{sec:GeneratingBoundedIterates}]
\label{main.lem:sigmastart}
If $\sigma = \widetilde O(\lambda_{min}(\Sigma_{pi})/d_p)$ and $\xi_{pi} \sim \mathcal N(0,\sigma\Sigma^{-1}_{pi})$ then $\| \xi_{pi} \|_2 \leq 1$ with very high probability.
\end{lemma}
Since initially $\Sigma_{p1} = I$, the above lemma determines the initial value  $\sigma \approx 1/d_p \ll \alpha_p$. This implies \Alg{} won't be able to overestimate the uncertainty $\mathcal U^\star_{pi}(\alpha_t)$ initially.

The solution is to have the algorithm proceed in epochs. At the end of every epoch \Alg{} ensures $\mathcal U^\star_{pi}(\sigma) \leq \epsilon$, and that $\lambda_{min}(\Sigma_{pi})$ is large enough that $\sigma$ can be doubled at the beginning of the next epoch.
\subsection{Learning an Epoch} 
Using 
\cref{main.lem:RelaxedProgram} we can analyze what happens within an epoch when $\sigma$ is fixed (assuming $\sigma$ is appropriately chosen to ensure $\| \xi_p \|_2 \leq 1$ with very high probability). We first consider the average uncertainty as a measure of progress and derive the bound below by neglecting the small error from encountering the feature $\phi_{pi}$ (step $(a)$ below) instead of the expected feature $\overline \phi_{\pi_i,p}$ (identified by the policy $\pi_i$ played by \Alg{} in episode $i$), by using a high probability bound $\| \xi_{pi} \|_{\Sigma_{pi}}\lessapprox \sqrt{d_p\sigma}$ and by using the elliptic potential lemma in \cite{Abbasi11} for the last step.
\begin{align}
\label{main.eqn:MainProofArgument}
& \frac{1}{k}\sum_{i=1}^{k} \mathcal U^\star_{pi}(\sigma) \overset{\substack{\text{\cref{main.lem:RelaxedProgram}}}}{\leq}
\frac{1}{k}\sum_{i=1}^{k} \overline \phi_{\pi_i,p}^\top \xi_{pi} \overset{(a)}{\approx} \frac{1}{k}\sum_{i=1}^{k} \phi_{pi}^\top \xi_{pi}  \overset{\substack{\text{Cauchy}\\\text{Schwartz}}}{\leq}  \frac{1}{k}\sum_{i=1}^{k} \|\phi_{pi} \|_{\Sigma^{-1}_{pi}} \overbrace{\| \xi_{pi} \|_{\Sigma_{pi}}}^{\lessapprox \sqrt{d_p\sigma}} \\
&  \overset{\substack{\text{Cauchy}\\\text{Schwartz}}}{\leq}\sqrt{\frac{d_p\sigma}{k}} \sqrt{\sum_{i=1}^k \| \phi_{pi} \|^2_{\Sigma^{-1}_{pi}}} \overset{\substack{\text{Elliptic }\\\text{Pot. Lemma}}}{\leq} d_p\sqrt{\frac{\sigma}{k}}.
\end{align}
The inequality $\overline \phi_{\pi_i,p}^\top\xi_{pi} \geq \mathcal U^\star_{pi}(\sigma)$ in the first step only holds for some of the episodes (since \cref{main.lem:RelaxedProgram} ensures the inequality with probability $q \in \R$), but this only affects the bound up to a constant with high probability. Since the uncertainty is monotonically decreasing, the last term $\mathcal U^\star_{pk}(\sigma)$ must be smaller than the average (the lhs of the above display), and we can conclude $\mathcal U^\star_{pk}(\sigma) \leq d_p\sqrt{\sigma/k}$.
Asking for the rhs to be $\leq \epsilon$ suggests we need $\approx d^2_p \sigma/\epsilon^2$ episodes. In essence, we have just proved the following:
\begin{lemma}[Number of trajectories to learn an epoch, see \cref{sec:LearningEpoch}]
In a given epoch \Alg{} ensures $\mathcal U^\star_{pk}(\sigma) \leq \epsilon$ with high probability using $\widetilde O(d^2_p \sigma/\epsilon^2)$ trajectories.
\end{lemma}

At the end of an epoch \Alg{} ensures $\mathcal U^\star_{pk}(\sigma) \leq \epsilon$,  but we really need $\mathcal U^\star_{pk} (\alpha_p) \leq \epsilon$ to hold. 

\subsection{Learning a Phase}
We need to use the explorability condition to allow \Alg{} to proceed to the next epoch:
\begin{lemma}[see \cref{sec:LearningALevel}]
Let $\underline k$ and $\overline k$ be the starting and ending episodes in an epoch.
If $\epsilon = \widetilde O(\nu_{min}/\sqrt{d_p})$ and $\mathcal U_{p\overline k}^\star(\sigma) = \widetilde O(\epsilon)$ then $\lambda_{min}(\Sigma_{p\overline k}) \geq 2\lambda_{min}(\Sigma_{p\underline k})$.
\end{lemma}

Since the minimum eigenvalue for the covariance matrix has doubled, we can double $\sigma$ (i.e., inject a stronger signal) and still satisfy \cref{main.lem:sigmastart}: at this point \Alg{} enters into a new epoch. At the beginning of every epoch we double $\sigma$, and this is repeated until $\sigma$ reaches the final value $\sigma \approx H^2 \alpha_p$. There are therefore only logarithmically many epochs (in the input parameters).
\begin{lemma}[\Alg{} meets target accuracy at the end of a phase, see \cref{sec:LearningALevel}]
	When \Alg{} reaches the end of the last epoch in phase $p$ it holds that $\sigma \approx H^2\alpha_p$ and $\epsilon \geq \mathcal U^\star_{p}(\sigma) = H\mathcal U^\star_{p}(\alpha_p)$. This implies $\mathcal U^\star_{p}(\alpha_p) \leq \epsilon/H$, as desired. Furthermore, this is achieved in $\widetilde O(d^2_pH^2\alpha_p/\epsilon)$ episodes.
\end{lemma}

Since $ \mathcal U^\star_{p}(\alpha_p) \leq \epsilon/H$ the inductive step is now proved; summing the number of trajectories over all the phases gives the final bound in \cref{main.thm:MainResult}. At this point, an $\epsilon$-optimal policy can be extracted by \lsvi{} on the returned dataset $\mathcal D$ for any prescribed linear reward function.

\subsection{Connection with G-optimal design}
\label{main.sec:ConnectionwithG-optimaldesign}
We briefly highlight the connection with G-optimal design. G-optimal design would choose a design matrix $\Sigma$ such that $\|\overline \phi_{\pi,p} \|_{\Sigma^{-1}}$ is as small as possible for all possible $\pi$. Since we cannot choose the features in the online setting, a first relaxation is to instead compute (and run) the policy $\pi$ that maximizes the program $\mathcal U^\star_{pi}(\sigma)$ in every episode $i$. Intuitively, as the area of maximum uncertainty is reached, information is acquired there and the uncertainty is progressively reduced, even though this might be not the most efficient way to proceed from an information-theoretic standpoint. Such procedure would operate in an online fashion, but unfortunately it requires an intractable optimization in policy space. Nonetheless this is the first relaxation to G-optimal design. To obtain the second relaxation, it is useful to consider the alternative definition $\mathcal U^\star_{pi}(\sigma) = \max_{\pi, \| \theta^{\mathcal U} \|_{\Sigma_{pi}} \leq \sqrt{\sigma}} \overline \phi_{\pi,p}^\top \theta^{\mathcal U}$. If we relax the constraint $\| \theta^{\mathcal U} \|_{\Sigma_{pi}} \leq \sqrt{\sigma}$ to obtain $\| \theta^{\mathcal U} \|_{\Sigma_{pi}} \lessapprox \sqrt{d_p\sigma}$ then the feasible space is large enough that random sampling from the feasible set (and computing the maximizing policy by using \lsvi{}) achieves the goal of overestimating the maximum of the unrelaxed program; in particular, sampling $\xi_{pi} \sim \mathcal N(0,\sigma\Sigma^{-1}_{pi})$ satisfies the relaxed constraints with high probability and is roughly uniformly distributed in the constraint set.

\section{Discussion}
This works makes progress in relaxing the optimistic closure assumptions on the function class for exploration through a statistically and computationally efficient PAC algorithm. From an algorithmic standpoint, our algorithm is inspired by \citep{osband2016generalization}, but from an analytical standpoint, it is justified by a design-of-experiments approach \citep{lattimore2020learning}. Remarkably, our approximations to make G-experimental design implementable \emph{online} and with polynomial \emph{computational complexity} only add a $d$ factor compared to G-optimal design. The proof technique is new to our knowledge both in principles and in execution, and can be appreciated in the appendix. We hope that the basic principle is general enough to serve as a foundation to develop new algorithms with even more general function approximators. The contribution to reward-free exploration \citep{jin2020rewardfree} to linear value functions is also a contribution to the field.

\section{Broader Impact}
This work is of theoretical nature and aims at improving our core understanding of reinforcement learning; no immediate societal consequences are anticipated as a result of this study.

\section*{Acknowledgment}

Funding in direct support of this work: Total Innovation Program Fellowship, ONR YIP and NSF career. The authors are grateful to the reviewers for their useful comments, in particular about the explorability requirement.
\small
\bibliographystyle{plainnat} 
\bibliography{rl.bib}

\begin{thebibliography}{44}
\providecommand{\natexlab}[1]{#1}
\providecommand{\url}[1]{\texttt{#1}}
\expandafter\ifx\csname urlstyle\endcsname\relax
  \providecommand{\doi}[1]{doi: #1}\else
  \providecommand{\doi}{doi: \begingroup \urlstyle{rm}\Url}\fi

\bibitem[Abbasi-Yadkori et~al.(2011)Abbasi-Yadkori, Pal, and
  Szepesvari]{Abbasi11}
Yasin Abbasi-Yadkori, David Pal, and Csaba Szepesvari.
\newblock Improved algorithms for linear stochastic bandits.
\newblock In \emph{Advances in Neural Information Processing Systems (NIPS)},
  2011.

\bibitem[Agarwal et~al.(2020)Agarwal, Henaff, Kakade, and Sun]{agarwal2020pc}
Alekh Agarwal, Mikael Henaff, Sham Kakade, and Wen Sun.
\newblock {PC-PG}: Policy cover directed exploration for provable policy
  gradient learning.
\newblock \emph{arXiv preprint arXiv:2007.08459}, 2020.

\bibitem[Agrawal and Jia(2017)]{Agrawal2017}
Shipra Agrawal and Randy Jia.
\newblock Optimistic posterior sampling for reinforcement learning: worst-case
  regret bounds.
\newblock In I.~Guyon, U.~V. Luxburg, S.~Bengio, H.~Wallach, R.~Fergus,
  S.~Vishwanathan, and R.~Garnett, editors, \emph{Advances in Neural
  Information Processing Systems (NIPS)}, pages 1184--1194. Curran Associates,
  Inc., 2017.

\bibitem[Azar et~al.(2017)Azar, Osband, and Munos]{Azar17}
Mohammad~Gheshlaghi Azar, Ian Osband, and Remi Munos.
\newblock Minimax regret bounds for reinforcement learning.
\newblock In \emph{International Conference on Machine Learning (ICML)}, 2017.

\bibitem[Chen and Jiang(2019)]{chen2019information}
Jinglin Chen and Nan Jiang.
\newblock Information-theoretic considerations in batch reinforcement learning.
\newblock In \emph{International Conference on Machine Learning (ICML)}, pages
  1042--1051, 2019.

\bibitem[Dann et~al.(2019)Dann, Li, Wei, and Brunskill]{dann2019policy}
Christoph Dann, Lihong Li, Wei Wei, and Emma Brunskill.
\newblock Policy certificates: Towards accountable reinforcement learning.
\newblock In \emph{International Conference on Machine Learning}, pages
  1507--1516, 2019.

\bibitem[Du et~al.(2019{\natexlab{a}})Du, Krishnamurthy, Jiang, Agarwal, Dudik,
  and Langford]{du2019provablyefficient}
Simon Du, Akshay Krishnamurthy, Nan Jiang, Alekh Agarwal, Miroslav Dudik, and
  John Langford.
\newblock Provably efficient {RL} with rich observations via latent state
  decoding.
\newblock In \emph{International Conference on Machine Learning (ICML)},
  volume~97, pages 1665--1674, Long Beach, California, USA, 09--15 Jun
  2019{\natexlab{a}}.

\bibitem[Du et~al.(2019{\natexlab{b}})Du, Kakade, Wang, and Yang]{du2019good}
Simon~S Du, Sham~M Kakade, Ruosong Wang, and Lin~F Yang.
\newblock Is a good representation sufficient for sample efficient
  reinforcement learning?
\newblock \emph{arXiv preprint arXiv:1910.03016}, 2019{\natexlab{b}}.

\bibitem[Du et~al.(2019{\natexlab{c}})Du, Luo, Wang, and Zhang]{du2019provably}
Simon~S Du, Yuping Luo, Ruosong Wang, and Hanrui Zhang.
\newblock Provably efficient q-learning with function approximation via
  distribution shift error checking oracle.
\newblock In \emph{Advances in Neural Information Processing Systems}, pages
  8058--8068, 2019{\natexlab{c}}.

\bibitem[Du et~al.(2020)Du, Lee, Mahajan, and Wang]{du2020agnostic}
Simon~S. Du, Jason~D. Lee, Gaurav Mahajan, and Ruosong Wang.
\newblock Agnostic q-learning with function approximation in deterministic
  systems: Tight bounds on approximation error and sample complexity.
\newblock \emph{arXiv preprint arXiv:2002.07125}, 2020.

\bibitem[Efroni et~al.(2019)Efroni, Merlis, Ghavamzadeh, and
  Mannor]{efroni2019tight}
Yonathan Efroni, Nadav Merlis, Mohammad Ghavamzadeh, and Shie Mannor.
\newblock Tight regret bounds for model-based reinforcement learning with
  greedy policies.
\newblock In \emph{Advances in Neural Information Processing Systems}, 2019.

\bibitem[Gopalan and Mannor(2015)]{gopalan2015thompson}
Aditya Gopalan and Shie Mannor.
\newblock Thompson sampling for learning parameterized markov decision
  processes.
\newblock In \emph{Conference on Learning Theory}, pages 861--898, 2015.

\bibitem[Hazan et~al.(2018)Hazan, Kakade, Singh, and Soest]{hazan2018provably}
Elad Hazan, Sham~M. Kakade, Karan Singh, and Abby~Van Soest.
\newblock Provably efficient maximum entropy exploration.
\newblock \emph{arXiv preprint arXiv:1812.02690}, 2018.

\bibitem[Jaksch et~al.(2010)Jaksch, Ortner, and Auer]{Jaksch10}
Thomas Jaksch, Ronald Ortner, and Peter Auer.
\newblock Near-optimal regret bounds for reinforcement learning.
\newblock \emph{Journal of Machine Learning Research}, 2010.

\bibitem[Jiang et~al.(2017)Jiang, Krishnamurthy, Agarwal, Langford, and
  Schapire]{jiang17contextual}
Nan Jiang, Akshay Krishnamurthy, Alekh Agarwal, John Langford, and Robert~E.
  Schapire.
\newblock Contextual decision processes with low {B}ellman rank are
  {PAC}-learnable.
\newblock In \emph{International Conference on Machine Learning (ICML)},
  volume~70, pages 1704--1713, International Convention Centre, Sydney,
  Australia, 06--11 Aug 2017.

\bibitem[Jin et~al.(2018)Jin, Allen-Zhu, Bubeck, and Jordan]{jin2018q}
Chi Jin, Zeyuan Allen-Zhu, Sebastien Bubeck, and Michael~I Jordan.
\newblock Is q-learning provably efficient?
\newblock In \emph{Advances in Neural Information Processing Systems}, pages
  4863--4873, 2018.

\bibitem[Jin et~al.(2020{\natexlab{a}})Jin, Krishnamurthy, Simchowitz, and
  Yu]{jin2020rewardfree}
Chi Jin, Akshay Krishnamurthy, Max Simchowitz, and Tiancheng Yu.
\newblock Reward-free exploration for reinforcement learning.
\newblock In \emph{International Conference on Machine Learning (ICML)},
  2020{\natexlab{a}}.

\bibitem[Jin et~al.(2020{\natexlab{b}})Jin, Yang, Wang, and
  Jordan]{jin2020provably}
Chi Jin, Zhuoran Yang, Zhaoran Wang, and Michael~I Jordan.
\newblock Provably efficient reinforcement learning with linear function
  approximation.
\newblock In \emph{Conference on Learning Theory}, 2020{\natexlab{b}}.

\bibitem[Kaufmann et~al.(2020)Kaufmann, M{\'e}nard, Domingues, Jonsson,
  Leurent, and Valko]{kaufmann2020adaptive}
Emilie Kaufmann, Pierre M{\'e}nard, Omar~Darwiche Domingues, Anders Jonsson,
  Edouard Leurent, and Michal Valko.
\newblock Adaptive reward-free exploration.
\newblock \emph{arXiv preprint arXiv:2006.06294}, 2020.

\bibitem[Kiefer and Wolfowitz(1960)]{kiefer1960equivalence}
Jack Kiefer and Jacob Wolfowitz.
\newblock The equivalence of two extremum problems.
\newblock \emph{Canadian Journal of Mathematics}, 12:\penalty0 363--366, 1960.

\bibitem[Krishnamurthy et~al.(2016)Krishnamurthy, Agarwal, and
  Langford]{krishnamurthy2016pac}
Akshay Krishnamurthy, Alekh Agarwal, and John Langford.
\newblock Pac reinforcement learning with rich observations.
\newblock In \emph{Advances in Neural Information Processing Systems (NIPS)},
  pages 1840--1848, 2016.

\bibitem[Lagoudakis and Parr(2003)]{lagoudakis2003least}
Michail~G Lagoudakis and Ronald Parr.
\newblock Least-squares policy iteration.
\newblock \emph{Journal of Machine Learning Research}, 4\penalty0
  (Dec):\penalty0 1107--1149, 2003.

\bibitem[Lattimore and Szepesv{\'a}ri(2020)]{lattimore2020bandit}
Tor Lattimore and Csaba Szepesv{\'a}ri.
\newblock \emph{Bandit Algorithms}.
\newblock Cambridge University Press, 2020.

\bibitem[Lattimore and Szepesvari(2020)]{lattimore2020learning}
Tor Lattimore and Csaba Szepesvari.
\newblock Learning with good feature representations in bandits and in rl with
  a generative model.
\newblock In \emph{International Conference on Machine Learning (ICML)}, 2020.

\bibitem[Lazaric et~al.(2012)Lazaric, Ghavamzadeh, and
  Munos]{lazaric2012finite}
Alessandro Lazaric, Mohammad Ghavamzadeh, and R{\'e}mi Munos.
\newblock Finite-sample analysis of least-squares policy iteration.
\newblock \emph{Journal of Machine Learning Research}, 13\penalty0
  (Oct):\penalty0 3041--3074, 2012.

\bibitem[M{\'e}nard et~al.(2020)M{\'e}nard, Domingues, Jonsson, Kaufmann,
  Leurent, and Valko]{menard2020fast}
Pierre M{\'e}nard, Omar~Darwiche Domingues, Anders Jonsson, Emilie Kaufmann,
  Edouard Leurent, and Michal Valko.
\newblock Fast active learning for pure exploration in reinforcement learning.
\newblock \emph{arXiv preprint arXiv:2007.13442}, 2020.

\bibitem[Misra et~al.(2020)Misra, Henaff, Krishnamurthy, and
  Langford]{misra2019kinematic}
Dipendra Misra, Mikael Henaff, Akshay Krishnamurthy, and John Langford.
\newblock Kinematic state abstraction and provably efficient rich-observation
  reinforcement learning.
\newblock In \emph{International Conference on Machine Learning (ICML)}, 2020.

\bibitem[Munos(2005)]{munos2005error}
R{\'e}mi Munos.
\newblock Error bounds for approximate value iteration.
\newblock In \emph{AAAI Conference on Artificial Intelligence (AAAI)}, 2005.

\bibitem[Munos and Szepesv{\'a}ri(2008)]{munos2008finite}
R{\'e}mi Munos and Csaba Szepesv{\'a}ri.
\newblock Finite-time bounds for fitted value iteration.
\newblock \emph{Journal of Machine Learning Research}, 9\penalty0
  (May):\penalty0 815--857, 2008.

\bibitem[Osband et~al.(2016{\natexlab{a}})Osband, Blundell, Pritzel, and
  Van~Roy]{osband2016deep}
Ian Osband, Charles Blundell, Alexander Pritzel, and Benjamin Van~Roy.
\newblock Deep exploration via bootstrapped {DQN}.
\newblock In \emph{Advances in Neural Information Processing Systems (NIPS)},
  2016{\natexlab{a}}.

\bibitem[Osband et~al.(2016{\natexlab{b}})Osband, Van~Roy, and
  Wen]{osband2016generalization}
Ian Osband, Benjamin Van~Roy, and Zheng Wen.
\newblock Generalization and exploration via randomized value functions.
\newblock In \emph{International Conference on Machine Learning (ICML)},
  2016{\natexlab{b}}.

\bibitem[Ouyang et~al.(2017)Ouyang, Gagrani, Nayyar, and
  Jain]{ouyang2017learning}
Yi~Ouyang, Mukul Gagrani, Ashutosh Nayyar, and Rahul Jain.
\newblock Learning unknown markov decision processes: A thompson sampling
  approach.
\newblock In \emph{Advances in Neural Information Processing Systems}, pages
  1333--1342, 2017.

\bibitem[Puterman(1994)]{puterman1994markov}
Martin~L. Puterman.
\newblock \emph{Markov Decision Processes: Discrete Stochastic Dynamic
  Programming}.
\newblock Wiley, 1994.

\bibitem[Russo(2019)]{russo2019worst}
Daniel Russo.
\newblock Worst-case regret bounds for exploration via randomized value
  functions.
\newblock In \emph{Advances in Neural Information Processing Systems}, 2019.

\bibitem[Sutton and Barto(2018)]{sutton2018reinforcement}
Richard~S Sutton and Andrew~G Barto.
\newblock \emph{Reinforcement learning: An introduction}.
\newblock MIT Press, 2018.

\bibitem[Tarbouriech et~al.()Tarbouriech, Pirotta, Valko, and
  Lazaric]{tarbouriechreward}
Jean Tarbouriech, Matteo Pirotta, Michal Valko, and Alessandro Lazaric.
\newblock Reward-free exploration beyond finite-horizon.
\newblock \emph{arXiv preprint arXiv:2002.02794}.

\bibitem[Wainwright(2019)]{wainwright2019high}
Martin~J Wainwright.
\newblock \emph{High-dimensional statistics: A non-asymptotic viewpoint},
  volume~48.
\newblock Cambridge University Press, 2019.

\bibitem[Wang et~al.(2020{\natexlab{a}})Wang, Du, Yang, and
  Salakhutdinov]{wang2020reward}
Ruosong Wang, Simon~S Du, Lin~F Yang, and Ruslan Salakhutdinov.
\newblock On reward-free reinforcement learning with linear function
  approximation.
\newblock \emph{arXiv preprint arXiv:2006.11274}, 2020{\natexlab{a}}.

\bibitem[Wang et~al.(2020{\natexlab{b}})Wang, Salakhutdinov, and
  Yang]{wang2020provably}
Ruosong Wang, Ruslan Salakhutdinov, and Lin~F. Yang.
\newblock Provably efficient reinforcement learning with general value function
  approximation.
\newblock \emph{arXiv preprint arXiv:2005.10804}, 2020{\natexlab{b}}.

\bibitem[Wang et~al.(2019)Wang, Wang, Du, and Krishnamurthy]{wang2019optimism}
Yining Wang, Ruosong Wang, Simon~S Du, and Akshay Krishnamurthy.
\newblock Optimism in reinforcement learning with generalized linear function
  approximation.
\newblock \emph{arXiv preprint arXiv:1912.04136}, 2019.

\bibitem[Yang and Wang(2020)]{yang2020reinforcement}
Lin~F Yang and Mengdi Wang.
\newblock Reinforcement leaning in feature space: Matrix bandit, kernels, and
  regret bound.
\newblock In \emph{International Conference on Machine Learning (ICML)}, 2020.

\bibitem[Zanette and Brunskill(2019)]{zanette2019tighter}
Andrea Zanette and Emma Brunskill.
\newblock Tighter problem-dependent regret bounds in reinforcement learning
  without domain knowledge using value function bounds.
\newblock In \emph{International Conference on Machine Learning (ICML)}, 2019.

\bibitem[Zanette et~al.(2020{\natexlab{a}})Zanette, Brandfonbrener, Pirotta,
  and Lazaric]{zanette2020frequentist}
Andrea Zanette, David Brandfonbrener, Matteo Pirotta, and Alessandro Lazaric.
\newblock Frequentist regret bounds for randomized least-squares value
  iteration.
\newblock In \emph{International Conference on Artificial Intelligence and
  Statistics (AISTATS)}, 2020{\natexlab{a}}.

\bibitem[Zanette et~al.(2020{\natexlab{b}})Zanette, Lazaric, Kochenderfer, and
  Brunskill]{zanette2020learning}
Andrea Zanette, Alessandro Lazaric, Mykel Kochenderfer, and Emma Brunskill.
\newblock Learning near optimal policies with low inherent bellman error.
\newblock In \emph{International Conference on Machine Learning (ICML)},
  2020{\natexlab{b}}.

\end{thebibliography}
\medskip

\appendix
\newpage
\tableofcontents

\newpage
\section{Preliminaries}
\subsection{Symbols}
\renewcommand{\arraystretch}{1.5}
\begin{longtable}{l l p{9cm}}
\caption{Symbols}\\
\hline
$r_t(s,a)$ & $ \defeq $& expected reward in $(s,a,t)$  \\
$p_t(s,a)$ & $ \defeq $& transition function in $(s,a,t)$  \\
$s_{tk}$ & $ \defeq $& experienced state at timestep $t$ in episode $k$ in phase $t$  \\
$a_{tk}$ & $ \defeq $& experienced action at timestep $t$  in episode $k$ in phase $t$  \\
$r_{tk}$ & $ \defeq $& experienced reward\footnote{this only applies if the reward function is learned from data; since we're doing reward free exploration, it instead represents the reward used to populate the dataset $\mathcal D$ after \Alg{} has terminated.} at timestep $t$  in episode $k$ in phase $t$  \\
$s^{+}_{t+1,k}$ & $ \defeq $& experienced state at timestep $t+1$ in episode $k$ in phase $t$  \\
$\Lphi$ & $ \defeq $& upper bound on $\sup_{s,a,t} \| \phi_t(s,a) \|_2$ \\
$\phi_{tk}$ & $ \defeq $ & $\phi_t(s_{tk},a_{tk})$ \\
$\Sigma_{tk}$ & $ \defeq $& $\sum_{i=1}^{k-1}\phi_{ti}\phi_{ti}^\top $ \\
$\Sigma_{t}$ & $ \defeq $& $\Sigma_{tk}$ matrix after \Alg{} has completed learning in phase $t$ ($k$ is the last episode in that phase) \\
$\T_t(Q_{t+1})(s,a) $ & $ \defeq $& $r_t(s,a) + \E_{s' \sim p_t(s,a)} Q_{t+1}(s,a)$ \\
$\T^P_t(Q_{t+1})(s,a) $ & $ \defeq $& $ \E_{s' \sim p_t(s,a)} Q_{t+1}(s,a)$ \\
$\mathring \theta_t (Q_{t+1}) $ & $\defeq$ & any  $\mathring \theta_t(Q_{t+1}) \in \mathcal B_t$ s.t. $\mathringdeffullBe$ when $Q_{t+1} \in \mathcal Q_{t+1}$   \\ 
$\Delta_{ti}(Q_{t+1})$ & $\defeq$ & $   \mathring Q_t(Q_{t+1})(s_{ti},\pi_{ti}(s_{ti})) - 	\T^P_t(Q_{t+1})(s_{ti},\pi_{ti}(s_{ti}))   $ \\
$ \widehat \theta_t $ & $\defeq $ &  $\Sigma_{t}^{-1} \sum_{i=1}^{n(t)} \phi_{ti} \big[ \Vhat_{t+1}(s^{+}_{t+1,i}) \big]$ \\
$\pi_{ti}$ & $\defeq$ & policy played in episode $i$ of phase $t$ \\
$Q_t(\theta)$ & $\defeq$ & action value function $(s,a) \mapsto  \phi_t(s,a)^\top\theta$ \\
$V_t(\theta)$ & $\defeq$ & value function $s \mapsto \max_a \phi_t(s,a)^\top\theta$ \\
$\eta^t_{ti}(\Vhat_{t+1})$ & $\defeq$ & $\etatdef$ \\
$ \Delta^r_t(s,a) $ & $\defeq$ & $ r_t(s,a) - \phi_t(s,a)^\top \theta^r_t$ \\
$ \Delta^r_{ti} $ & $\defeq$ & $\DeltaRdef$ \\
$\eta^{r}_{ti}$ & $\defeq$ & $\etardef$ (reward noise) \\
$\IBEE(\mathcal Q_{t},\mathcal Q_{t+1}) $ & $\defeq$ & $\max_{\substack{Q_{t+1}\in \mathcal Q_{t+1}}} \min_{Q_t \in \mathcal Q_{t}} \max_{(s,a)} |[Q_t - \T_t^P(Q_{t+1})](s,a)|$ \\
$E_t$ & $\defeq$ & approximation error for the reward, see \cref{eqn:174} \\
$k$ & $\defeq$ & is an overestimate\footnote{in particular it can be set to be equal to $n(t)$ and is $poly(d_1,\cdots,d_H,H,\frac{1}{\epsilon},\frac{1}{\delta})$} of the number of episodes and is used in the definition of the $\beta$'s  below \\ 
$\delta'$ & $\defeq$ & is\footnote{in particular it is $\frac{\delta}{poly(d_1,\cdots,d_H,H,\frac{1}{\epsilon},\frac{1}{\delta})}$}  used in the definition of the $\beta$'s  below \\ 
$\beta^t_{t} $ & $\defeq$ & $\sqrtbetadef $ \\
$\beta_t^{r} $ & $\defeq$ & $ \sqrt{d_t \ln \(\frac{1+k\Lphi^2}{\delta'} \) } + \| \theta^r_t \|_2 $  \\
$D_{p} $ &  $\defeq$ & $ d_p \ln(1+k\Lphi^2/d_p)$ \\
$\beta_t^{E} $ & $\defeq$ & $ \beta_t^{r} $  \\
$ \sqrt{\alpha_t}$ & $\defeq$ & $ 3\(\sqrt{\beta^t_{t}} + \sqrt{\beta_t^{r}} + 2\)  = \widetilde O(\sqrt{d_t + d_{t+1}})$ \\
$ n(t) $ & $\defeq$ & number of samples collected in phase $t$ \\
$ \widehat \theta_t $ & $\defeq$ & $ \Sigma_{t}^{-1} \sum_{i=1}^{n(t)} \phi_{ti} \big[ \Vhat_{t+1}(s^{+}_{t+1,i}) \big] $ \\
$ \widehat \theta^r_t $ & $\defeq$ & $ \Sigma_{t}^{-1} \sum_{i=1}^{n(t)} \phi_{ti} \big[ r_{tk} \big] $ \\
$ \widehat \theta^{R+PV}_t $ & $\defeq$ &  $ \widehat \theta^{r}_t + \widehat \theta_t$ \\
$ \Radius_t $ & $\defeq$ &  radius at timestep $t$ (but these will be all equal to $1$ in the end) \\
$ \Radius $ & $\defeq$ &  $\Radius_1 = \dots = \Radius_H = 1$ \\
$ q $ & $\defeq$ &  $\Phi(-3)$ (normal cdf evaluated at $-3$) \\
$\mathcal C_k$ & $\defeq$ & $\Ckdef$ \\
$\mathcal E_k$ & $\defeq$ & $\Ekdef$ \\
$k(e,i)$ & $\defeq $ & episode in epoch $e$ (of a certain phase) such that $E_k$ happens for the $i$-th time. \\
	$\zeta_{pk(e,i)}$ & $ \defeq$ & $ \zetadef$\\
$ A $ & $\defeq$ & $ \Adef $  \\
$ \gamma_t(\sigma) $ & $\defeq$ & $ \twonormbound $ \\
$ \pi_t(s) $ & $\defeq$ & indicates the action taken at timestep $t$ by policy $\pi$ in state $s$ \\
$ \sigma_{Start}$ & $\defeq$ & $1/\( \MagicValue{p}\)$ \\
$a\mathcal B_t$ & $\defeq$ & $ \{ ax \mid x \in \mathcal B_t \}$ for a positive real $a$ \\
    $V_t^\pi $ & $\defeq$ & value function of policy $\pi$ at timestep $t$ on $\mathcal M$ once the reward function is fixed \\
        $\Vstar $ & $\defeq$ & optimal value function on $\mathcal M$ once the reward function is fixed \\
            $\pistar $ & $\defeq$ & optimal policy on $\mathcal M$ once the reward function is fixed \\
$c_e,c_{\alpha},c_\sigma$ & $\defeq$ & constants implicitly determined, see proof of \cref{thm:MainResult} and footnote in that page 
\label{tab:MainNotation}
\end{longtable}

\newpage
\subsection{Inherent Bellman Error}
\label{sec:IBE}

\begin{definition}[Inherent Bellman Error and Best Approximator]
\label{def:InherentBellmanError}
Given two compact linear functional spaces\footnote{For infinite horizon MDPs, these normally coincide.} $\mathcal Q_{t}$ and $\mathcal Q_{t+1}$, the inherent Bellman error at step $t$ is the maximum (in absolute value) residual
\begin{align*}
	\IBEE(\mathcal Q_t,\mathcal Q_{t+1}) & \defeq \max_{\substack{Q_{t+1}\in \mathcal Q_{t+1}}} \min_{Q_t \in \mathcal Q_{t}} \max_{(s,a)} |[Q_t - \T_t^P(Q_{t+1})](s,a)|.
	\end{align*}

The approximator $\mathring Q_{t}( Q_{t+1}) \in \mathcal Q_{t}$ of $Q_{t+1} \in \mathcal Q_{t+1}$ through $\T_t^P$ is defined by its parameter $\mathring \theta_t (Q_{t+1})$ as any solution $\theta_t \in \mathcal B_t$ that verifies (this always exists from the above display) for any $Q_{t+1} \in \mathcal Q_{t+1}$
\begin{align}
\label{eqn:mathringQdef}
 \mathringdeffullBe
\end{align}
The Bellman residual function $\overline \Delta_{\pi,t}$ under policy $\pi$ is implicitly defined in the error decomposition below:
\begin{align}
\label{eqn:BellmanErrorDecomposition}
 \T^P_t(Q_{t+1})(s,a) \defeq \mathring Q_t(Q_{t+1})(s,a)  + \Delta_{t}(Q_{t+1})(s,a). 
\end{align}
and it satisfies
\begin{align}
	\IBEE(\mathcal Q_t,\mathcal Q_{t+1}) = \max_{\substack{(s,a) \\ Q_{t+1}\in \mathcal Q_{t+1}}} | \Delta_{t}(Q_{t+1})(s,a)|
\end{align}

\end{definition}

We briefly argue why we have the last equality in the above definition

\begin{align}
	\IBEE(\mathcal Q_{t},\mathcal Q_{t+1}) & \geq \max_{\substack{Q_{t+1}\in \mathcal Q_{t+1}}} \max_{(s,a)} |  \mathring Q_t(Q_{t+1})(s,a) - \T_t^P(Q_{t+1})(s,a)| \\
	& = \max_{\substack{Q_{t+1}\in \mathcal Q_{t+1}}} \max_{(s,a)} | \Delta_{t}(Q_{t+1})(s,a) |
\end{align}
where the second step uses \cref{eqn:BellmanErrorDecomposition}.

We are going to use the following property throughout the appendix:
\begin{proposition}[Positive Homogeneity of Inherent Bellman Error of System Dynamics]
\label{prop:PositiveHomogeneity} Let $\gamma$ be a positive scalar number.
If 
\begin{align}
\max_{\substack{Q_{t+1}\in \mathcal Q_{t+1}}} \min_{Q_t \in \mathcal Q_{t}} \max_{(s,a)} |[Q_t - \T_t^P(Q_{t+1})](s,a)| \leq \IBEE(\mathcal Q_{t},\mathcal Q_{t+1})
\end{align}
then
\begin{align}
\label{eqn:Statement1}
	\max_{\substack{Q_{t+1}\in \gamma \mathcal Q_{t+1}}} \min_{ Q_t \in \gamma \mathcal Q_{t}} \max_{(s,a)} |[Q_t - \T_t^P(Q_{t+1})](s,a)| \leq \gamma \IBEE(\mathcal Q_{t},\mathcal Q_{t+1})
\end{align}
where
\begin{align}
\label{eqn:14}
	\gamma \mathcal Q_\tau = \{ Q_\tau \mid Q_\tau(s,a) = \phi_\tau(s,a)^\top\theta, \| \theta \|_2 \leq \gamma \Radius_\tau \}, \quad \tau \in \{t,t+1\}.
\end{align}
This implies that if $\|\theta_{t+1}\|_2 \leq \gamma\Radius_{t+1}$ then we can find a $\mathring \theta_t$ satisfying $\| \mathring \theta_{t}(V_{t+1}(\theta_{t+1}))\|_2 \leq \gamma \Radius_t$.
\end{proposition}
\begin{proof}
Notice that when we write $\max_x f(x) \leq I$ (for a generic scalar function $f$, an element $x$ in a set, and a scalar $I$) we can replace the statement with $\forall x, f(x) \leq I$ and viceversa:
\begin{align}
	\max_x f(x) \leq I \longleftrightarrow \forall x, \; f(x) \leq I 
\end{align}
Likewise:
\begin{align}
	\max_x \min_y f(x,y) \leq I \longleftrightarrow \forall x, \; \exists y: \; f(x,y) \leq I 
\end{align}
We can recast the Bellman error condition as
\begin{align}
\forall Q_{t+1}\in \mathcal Q_{t+1}, \;  \exists Q_t \in \mathcal Q_{t}:  \max_{(s,a)} |[Q_t - \T_t^P(Q_{t+1})](s,a)| \leq \IBEE(\mathcal Q_{t},\mathcal Q_{t+1})
\end{align}

Now consider the bijection
\begin{align*}
	Q_{t}\in \mathcal Q_{t} & \longleftrightarrow  Q'_t = \gamma Q_{t}\in \gamma \mathcal Q_{t}, \quad  \\
			Q_{t+1}\in \mathcal Q_{t+1} & \longleftrightarrow Q'_{t+1} = \gamma Q_{t+1}\in \gamma \mathcal Q_{t+1}, \quad  \\
\numberthis{\label{eqn:Bijections}} 
\end{align*}
We have that the statement below
\begin{align}
\label{eqn:51}
\forall Q'_{t+1}\in \gamma \mathcal Q_{t+1}, \;  \exists Q'_t \in \mathcal \gamma Q_{t}:  \max_{(s,a)} |[Q_t - \T_t^P(Q_{t+1})](s,a)| \leq \gamma \IBEE(\mathcal Q_{t},\mathcal Q_{t+1})
\end{align}
holds if and only if
\begin{align}
\label{eqn:52}
\forall Q_{t+1}\in \mathcal Q_{t+1}, \;  \exists Q_t \in \mathcal  Q_{t}:  \max_{(s,a)} |[\gamma Q_t - \T_t^P(\gamma Q_{t+1})](s,a)| \leq \gamma \IBEE(\mathcal Q_{t},\mathcal Q_{t+1})
\end{align}
holds. Therefore, it suffices to prove \cref{eqn:52} to prove the statement.
Notice that by linearity of expectation for any $\gamma > 0$ we have
\begin{align}
\T_t^P  Q_{t+1}(\gamma\theta_{t+1}))(s,a) & = \E_{s' \sim p_t(s,a)} \max_{a'}[\gamma Q_{t+1}(\theta_{t+1})(s',a')] ]\\
& =
	\gamma \E_{s' \sim p_t(s,a)} \max_{a'}[ Q_{t+1}(\theta_{t+1})(s',a')] \\
	& = \gamma \T_t^P (Q_{t+1})(\theta_{t+1})(s,a).
\end{align}
Therefore
\begin{align}
\max_{(s,a)} |[\gamma Q_t - \T_t^P(\gamma Q_{t+1})](s,a)| = \gamma \max_{(s,a)} |[ Q_t - \T_t^P( Q_{t+1})](s,a)|  
\end{align}
The hypothesis of the lemma implies
\begin{align}
	\forall Q_{t+1}\in \mathcal Q_{t+1}, \;  \exists Q_t \in \mathcal  Q_{t}: \gamma \max_{(s,a)} |[ Q_t - \T_t^P( Q_{t+1})](s,a)| \leq \gamma  \IBEE(\mathcal Q_{t},\mathcal Q_{t+1})
\end{align}
and the prior display implies that \cref{eqn:52} holds, and so does 
\cref{eqn:51} which is equivalent to \cref{eqn:Statement1}. 

Finally to conclude the proof of the theorem notice that if 
$
	\theta_{t+1} \in \gamma \Radius_{t+1}
$
then we can find a $
	\mathring \theta_{t} \in \gamma \Radius_{t}
$
such that the Bellman error is at most $\gamma \IBEE(\mathcal Q_t,\mathcal Q_{t+1})$.
\end{proof}

\newpage
\section{Analysis of vanilla \lsvi}
\label{sec:lsvi}
We recall the popular \lsvi{} protocol \citep{munos2005error,munos2008finite} operating on a \emph{batch} dataset $\mathcal D = \{ \(s_{tk},a_{tk},r_{tk},s^+_{t+1,k} \) \}^{t=1,\dots,H}_{k=1,\dots,n(t)}$ of experienced state-action-reward-successor states. We use $n(t)$ to denote the number of samples collected at a certain timestep $t$. The regularization parameter is optional and defaults to $\lambda = 1$. The \lsvi{} algorithm is used without reward from the dataset $\mathcal D$ when called by \Alg{}; instead a pseudoreward function $\textsc{r}_p$ is prescribed in the last timestep.
\dumplsvi

\subsection{Single Step Error Decomposition}
\begin{lemma}[Analysis of Transition Error in Parameter Space]
\label{lem:InParameterSpace}
Let $n(t)$ be the number of episodes where samples have been collected at timestep $t$.
If $\widehat \theta_{t}$ satisfies
\begin{align}
\label{eqn:ls}
 \widehat \theta_t = \Sigma_{t}^{-1} \sum_{i=1}^{n(t)} \phi_{ti} \big[ \Vhat_{t+1}(s^{+}_{t+1,i}) \big] 
\end{align}
then it must also satisfy:
\begin{align}
\label{eqn:firstequation}
\widehat \theta_t = \mathring \theta_t(\Vhat_{t+1}) + \Sigma_{t}^{-1} \( \sum_{i=1}^{n(t)}  \phi_{ti}\big[ \Delta_{ti}(\Vhat_{t+1}) + \eta^{t}_{ti}(\Vhat_{t+1})\big] -\lambda\mathring \theta_t(\Vhat_{t+1}) \) .
\end{align}
\end{lemma}
\begin{proof}
Let $\pi_{ti}$ be the policy used to generate the rollouts of episode $i$ of phase $t$. Define the trajectory noise of episode $i$ of phase $t$  using the next-state value function $\Vhat_{t+1}$ as:
\begin{align}
\eta^t_{ti}(\Vhat_{t+1}) \defeq \etatdef.
\end{align}
From \cref{eqn:ls} we can rewrite the unique solution for $\widehat \theta_t$ as
\begin{align*}
	 \widehat \theta_t & = \Sigma_{t}^{-1} \sum_{i=1}^{n(t)} \phi_{ti} \big[\E_{s'\sim p(s_{ti},\pi_{ti}(s_{ti}))}\Vhat_{t+1}(s') + \eta^{t}_{ti}(\Vhat_{t+1}) \big]
\numberthis{\label{eqn:ExpansionBegin}}
\end{align*}
Recall the error decomposition of \cref{eqn:BellmanErrorDecomposition} with $(s,a) = (s_{ti},\pi_{ti}(s_{ti})), \phi_{ti} = \phi(s,a), \Delta_{ti} = \Delta_t(s,a)$
\begin{align*}
		\E_{s'\sim p(s,a)}\Vhat_{t+1}(s') & =  \phi_{ti}^\top \mathring \theta_t(Q_{t+1}) + \Delta_{ti}(Q_{t+1})
		 \numberthis{\label{eqn:etaV}}
\end{align*}
where $\mathring \theta_t(Q_{t+1}) \in \mathcal B_t$.

Plugging back \cref{eqn:etaV} into \cref{eqn:ExpansionBegin} gives:
\begin{align}
	\widehat \theta_t & = \Sigma_{t}^{-1} 
\( \sum_{i=1}^{n(t)} \phi_{ti} \big[ \phi_{ti}^\top \mathring \theta_{t}(\Vhat_{t+1}) + \Delta_{ti}(\Vhat_{t+1}) + \eta^{t}_{ti}(\Vhat_{t+1})\big] + \overbrace{\lambda\mathring \theta_t(\Vhat_{t+1}) -\lambda\mathring \theta_t(\Vhat_{t+1})}^{=0} \) \\ 
	& = \Sigma_{t}^{-1} \Sigma_{t}\mathring \theta_t(\Vhat_{t+1}) + \Sigma_{t}^{-1} 
	\(\sum_{i=1}^{n(t)} \phi_{ti}\big[ \Delta_{ti}(\Vhat_{t+1}) + \eta^{t}_{ti}(\Vhat_{t+1}) \big] -\lambda\mathring \theta_t(\Vhat_{t+1})\) \\
	& = \mathring \theta_t(\Vhat_{t+1}) + \Sigma_{t}^{-1} \( \sum_{i=1}^{n(t)} \phi_{ti}\big[ \Delta_{ti}(\Vhat_{t+1}) + \eta^{t}_{ti}(\Vhat_{t+1}) \big] -\lambda\mathring \theta_t(\Vhat_{t+1})\) .
\end{align}
This proves the lemma.
\end{proof}

\newpage
\begin{lemma}[Analysis of Reward Error in Parameter Space]
\label{lem:InParameterSpaceReward}
Let $n(t)$ be the number of episodes where samples have been collected at timestep $t$.
If $\widehat \theta^r_{t}$ satisfies
\begin{align}
\label{eqn:lsr}
 \widehat \theta^r_t = \Sigma_{t}^{-1} \sum_{i=1}^{n(t)} \phi_{ti} r_{tk} 
\end{align}
then it must also satisfy:
\begin{align}
\label{eqn:firstequationreward}
\widehat \theta^r_t & = \theta^r_t + \Sigma_{t}^{-1} \( \sum_{i=1}^{n(t)} \phi_{ti} \big[  \eta^r_{ti} +  \Delta^r_{ti} \big] - \lambda \theta^r_t \)
\end{align}

\end{lemma}
\begin{proof}
Let $\pi_{ti}$ be the policy used to generate the rollouts of episode $i$ of phase $t$. 

From \cref{eqn:lsr} we can rewrite the unique solution for $\widehat \theta^r_t$ as (for the definitions of the symbols see \cref{tab:MainNotation})
\begin{align*}
	 \widehat \theta^r_t & = \Sigma_{t}^{-1} \sum_{i=1}^{n(t)} \phi_{ti} \big[ r_t(s_{ti},a_{ti}) + \eta^r_{ti} \big] \\ 
	& =  \Sigma_{t}^{-1} \( \sum_{i=1}^{n(t)} \phi_{ti} \big[ \phi_{ti}^\top \theta^r_{t} + \Delta^r_{ti} + \eta^r_{ti} \big]   + \lambda \theta^r_t - \lambda \theta^r_t \) \\
	& =  \theta^r_t + \Sigma_{t}^{-1} \( \sum_{i=1}^{n(t)} \phi_{ti} \big[ \eta^r_{ti} + \Delta_{ti} \big] - \lambda \theta^r_t \) 
\numberthis{\label{eqn:ExpansionBeginReward}}
\end{align*}
\end{proof}

\newpage
\subsection{Single Step Error Bounds}
\begin{definition}[Good Event for \lsvi{}]
\label{def:GoodEventLSVI}
Assume $\sqrt{n(t)}\IBEE(\mathcal Q_{t},\mathcal Q_{t+1}) \leq \sqrt{\alpha_t}/3$ and $\sqrt{n(t)}E_t \leq \sqrt{\alpha_t}/3$.
We say that \lsvi{} (\cref{alg:lsviFrancis,alg:lsviBatch}) is in the good event when the following bound holds for all $t \in [H]$ with\footnote{Note that if $\widehat V_{t+1} \in R\times\mathcal V_{t+1}$ (the set $\mathcal V_{t+1}$ where all elements are scaled by the scalar $R$) then the bounds still hold provided that they are rescaled by $R$.} $\Vhat_{t+1} \in \mathcal V_{t+1}$. The definition of the symbols are reported in \cref{tab:MainNotation}:
\begin{align}
	\| \sum_{i=1}^{n(t)} \phi_{ti} 
	 \Delta_{ti}(\Vhat_{t+1}) \|_{\Sigma_{t}^{-1}}	& \leq \sqrt{n(t)}\IBEE(\mathcal Q_{t},\mathcal Q_{t+1}) \\
		\| \sum_{i=1}^{n(t)} \phi_{ti} \eta^t_{ti}(\Vhat_{t+1}) \|_{\Sigma_{t}^{-1}} & \leq \sqrt{\beta^t_{t}} \\
		\lambda\| \mathring \theta_t(\Vhat_{t+1}) \|_{\Sigma_{t}^{-1}} & \leq \sqrt{\lambda} \Radius_t \\
			\| \sum_{i=1}^{n(t)} \phi_{ti} 
	 \Delta^r_{ti} \|_{\Sigma_{t}^{-1}}	& \leq \sqrt{n(t)}E_t \\
		\| \sum_{i=1}^{n(t)} \phi_{ti} \eta^r_{ti} \|_{\Sigma_{t}^{-1}} & \leq \sqrt{\beta^{r}_{t}} \\
		\lambda\|  \theta^r_t \|_{\Sigma_{t}^{-1}} & \leq \sqrt{\lambda} \|\theta^r_t \|_2.
\end{align}

In addition, the above expressions with the relations in \fullref{lem:InParameterSpace} and 
 \fullref{lem:InParameterSpaceReward} imply:

\begin{align*}
	& \| \widehat \theta^r_t - \theta^r_t\|_{\Sigma_{t}} + \| \widehat \theta_t - \mathring \theta_t(\Vhat_{t+1}) \|_{\Sigma_{t}}  \\
	& \leq \sqrt{n(t)}\IBEE(\mathcal Q_{t},\mathcal Q_{t+1}) + \sqrt{n(t)}E_t +\sqrt{\beta^r_{t}} +\sqrt{\beta^t_{t}} + \sqrt{\lambda} \Radius_t + \sqrt{\lambda} \|\theta^r_t \|_2 \\
	& \leq \sqrt{\alpha_t}
\numberthis{\label{eqn:next}}
\end{align*}
\end{definition}

\begin{lemma}[Probability of Good Event for \lsvi{}] 
\label{lem:ErrorBounds}
There exists a parameter $\delta' = \frac{\delta}{poly(d_1,\dots,d_H,H,\frac{1}{\epsilon})}$, such that the good event of \cref{def:GoodEventLSVI} holds with probability at least $1-\delta/2$.
\end{lemma}

\begin{proof}
Since $|\Delta_{ti}(\Vhat_{t+1}) | \leq  \IBEE(\mathcal Q_{t},\mathcal Q_{t+1})$, the projection bound (lemma 8 in \citep{zanette2020learning}) gives the first inequality in the statement of the theorem.
The second inequality is proved in \fullref{lem:TransitionsHPbound}  respectively. The third inequality follows from \fullref{lem:MaxEvaluBound}. Since $| \Delta^r_{ti} | \leq  E_t$ the projection bound (lemma 8 in \citep{zanette2020learning}) again gives the fourth inequality. The fifth inequality follows from theorem 2 in \citep{Abbasi11}  with $1$-sub-Gaussian noise and the last inequality again follows from \fullref{lem:MaxEvaluBound}. 
In particular it is possible to choose $\delta'$ (in the definition of the $\beta$'s) such that these statements jointly hold with probability at least $1-\delta/2$ after a union bound over each statement and the timestep $H$. At this point the statement in \cref{eqn:next} follows deterministically by chaining with \cref{lem:InParameterSpace,lem:InParameterSpaceReward}.
\end{proof}

\newpage
\subsection{Iterate Boundness}

In this section we discuss the \emph{boundness} in the value function parameter. 

\begin{lemma}[Boundness at Intermediate Timesteps for \cref{alg:lsviFrancis}]
\label{lem:IntermediateBoundnesFrancis}
On the good event for \lsvi{} of \cref{def:GoodEventLSVI} if \begin{align}
	\lambda_{min}(\Sigma_{t}) & \geq 4H^2\alpha_t, \quad \forall t \in [p-1] \\
	\| \xi_p \|_2 & \leq \frac{1}{2} 
\end{align}
then
\begin{align}
	\| \widehat \theta_t \|_2 \leq 1, \quad \forall t \in [p].
\end{align}
\end{lemma}
\begin{proof}
We proceed by induction, showing that $\widehat \theta_t$ due to errors can live in bigger and bigger balls, with radius starting from $\frac{1}{2}$ at timestep $p$ to radius  $1$ at timestep $1$.
\begin{inductivehypothesis}
	$
	\| \widehat \theta_t\|_2 \leq  (1 - \frac{t-1}{2H})  $.
\end{inductivehypothesis}	
The inductive statement clearly holds at $t = p$ by hypothesis of the lemma; therefore we focus on the inductive step (notice that the induction goes from $t=p$ down to $t=1$, so the inductive step assumes the inductive hypothesis holds when written for $t+1$.)

The inherent Bellman error definition (\fullref{def:InherentBellmanError}) and \fullref{prop:PositiveHomogeneity} ensures
\begin{align}
	\| \widehat \theta_{t+1} \|_2 \leq  \(1 - \frac{t}{2H}\)   \longrightarrow	 	\| \mathring \theta_t(V_{t+1}(\widehat \theta_{t+1})) \|_2 \leq  \(1 - \frac{t}{2H}\) \end{align}
In particular, the left statement is ensured by the inductive hypothesis for $t+1$.
Next,  under the good event of \fullref{def:GoodEventLSVI}, we have that \fullref{lem:MaxEvaluBound} ensures (writing $ \mathring \theta_t = \mathring \theta_t(V_{t+1}(\widehat \theta_{t+1}))$ for short)
\begin{align}
\sqrt{\alpha_t} \geq \| \widehat \theta_t - \mathring \theta_t \|_{\Sigma_{t}} &  \geq \sqrt{\lambda_{min}(\Sigma_{t})} \| \widehat \theta_t - \mathring \theta_t \|_2
\end{align}
Solving for $ \| \widehat \theta_t - \mathring \theta_t \|_2$ and using the lemma's hypothesis gives
\begin{align}
	 \| \widehat \theta_t - \mathring \theta_t \|_2 \leq \frac{\sqrt{\alpha_t}}{2H\sqrt{\alpha_t}} = \frac{1}{2H}.
\end{align}
Combined with the prior display, we deduce
\begin{align}
	\| \widehat \theta_t\|_2 \leq \| \widehat \theta_t - \mathring \theta_t \|_2  + \|\mathring \theta_t \|_2 \leq 1 - \frac{t}{2H} + \frac{1}{2H} = 1 - \frac{t-1}{2H}.
\end{align}
This shows the inductive step.
\end{proof}

\newpage
\begin{lemma}[Boundness at Intermediate Timesteps for \cref{alg:lsviBatch}]
\label{lem:IntermediateBoundnes}
Under the good event \cref{def:GoodEventLSVI}, fix a positive scalar $R$; if
\begin{align}
	\lambda_{min}(\Sigma_{t}) & \geq 4H^2\alpha_t, \quad \forall t \in [H] \\
	\| \theta^r_t \|_2 & \leq \frac{R}{H} 
\end{align}
then
\begin{align}
	\| \widehat \theta_t^{R+PV} \|_2 = \| \widehat \theta^R_t +  \widehat \theta_t \|_2 \leq 2R, \quad \forall t \in [H].
\end{align}
\end{lemma}
\begin{proof}
We proceed by induction, showing that $\widehat \theta_t^{R+PV}$ due to errors can live in bigger and bigger balls
\begin{inductivehypothesis}
	$
	\|\widehat \theta_t^{R+PV} \|_2 \leq  2(1 - \frac{t-1}{H})R  $.
\end{inductivehypothesis}	
The inductive statement clearly holds at $t = H+1$; therefore we focus on the inductive step (notice that the induction goes from $t=H+1$ down to $t=1$, so the inductive step assumes the inductive hypothesis holds when written for $t+1$).

The inherent Bellman error definition (\fullref{def:InherentBellmanError}) and \fullref{prop:PositiveHomogeneity} ensures
\begin{align}
	\| \widehat \theta^{R+PV}_{t+1} \|_2 \leq  2\(1 - \frac{t}{H}\)R   \longrightarrow	 	\| \mathring \theta_t(V_{t+1}(\widehat \theta^{R+PV}_{t+1})) \|_2 \leq  2\(1 - \frac{t}{H}\)R \end{align}
In particular, the left statement is ensured by the inductive hypothesis for $t+1$.
Next, under the good event of \fullref{def:GoodEventLSVI} (with a scaling argument by $R$ on the $\| \cdot \|_2$ norm of the regressed parameter) we have that  \fullref{lem:MaxEvaluBound} ensures (writing $ \mathring \theta_t = \mathring \theta_t(V_{t+1}(\widehat \theta^{R+PV}_{t+1}))$ for short)
\begin{align}
R\sqrt{\alpha_t} \geq \( \| \widehat \theta^r_t - \theta^r_t \|_{\Sigma_t} +  \| \widehat \theta_t - \mathring \theta_t \|_{\Sigma_t} \) &  \geq \sqrt{\lambda_{min}(\Sigma_{t})} \( \| \widehat \theta^r_t - \theta^r_t \|_2 +  \| \widehat \theta_t - \mathring \theta_t \|_2 \)
\end{align}
Solving for $ \( \| \widehat \theta^r_t - \theta^r_t \|_2 +  \| \widehat \theta_t - \mathring \theta_t \|_2 \)$ and using the lemma's hypothesis gives
\begin{align}
	 \( \| \widehat \theta^r_t - \theta^r_t \|_2 +  \| \widehat \theta_t - \mathring \theta_t \|_2 \)\leq \frac{\sqrt{\alpha_t}}{2H\sqrt{\alpha_t}}R \leq \frac{R}{2H}.
\end{align}
Combined with the prior display, we deduce
\begin{align}
	\| \widehat \theta^{R+PV}_t\|_2 \leq \| \widehat \theta^{r}_t - \theta^r_t \|_2 + \| \widehat \theta_t - \mathring \theta_t \|_2  +  \| \theta^r_t \|_2  +  \|\mathring \theta_t \|_2 \leq \frac{R}{2H} + \frac{R}{H} + 2(1 - \frac{t}{H})R \leq  2(1 - \frac{t-1}{H})R .
\end{align}
This shows the inductive step.
\end{proof}

\newpage
\subsection{Multi-Step Analysis: Error Bounds for \lsvi{}}
\label{sec:lsvi-final-bound}

\begin{lemma}[Telescopic Expansion]
	\label{lem:Telescope}
	Under the good event of \cref{def:GoodEventLSVI} for \cref{alg:lsviFrancis} if
\begin{align}
\| \xi_p \|_2 \leq \frac{1}{2}
\end{align}	
then the learned parameter
	\begin{align}
		\|\widehat \theta_t \|_2 \leq 1, \quad t \in [p].
	\end{align}
	Furthermore, for any policy $\pi$
\begin{align}
	 \E_{x_1 \sim \rho}\Qhat_{1}(x_1,\pi_1(x_1)) \geq - \sum_{t=1}^{p-1} \Big[ \IBEE(\mathcal Q_{t},\mathcal Q_{t+1}) + \sqrt{\alpha_{t}}\|\overline \phi_{\pi,t}\|_{\Sigma_{t}^{-1}} \Big] + \E_{x_{p} \sim \pi}\Qhat_{p}(x_{p},\pi_{p}(x_{p}))
\end{align}
and for the greedy policy $\overline \pi$ with respect to $\Qhat$, i.e., $\overline \pi_t(s) = \argmax_a \Qhat_t(s,a)$
it additionally holds that 
\begin{align}
	\E_{x_1 \sim \rho}\Vhat_{1}(x_1) & \leq \sum_{t=1}^{p-1} \Big[ \IBEE(\mathcal Q_{t},\mathcal Q_{t+1}) + \sqrt{\alpha_{t}} \|\overline \phi_{\overline \pi,t}\|_{\Sigma_{t}^{-1}} \Big] + \E_{x_{p} \sim \overline \pi}\Vhat_{p}(x_{p}).
\end{align}
\end{lemma}
\begin{proof}
On the good event for \lsvi{} of \fullref{def:GoodEventLSVI} the boundness of the iterate $\widehat \theta_t$ is given by \fullref{lem:IntermediateBoundnesFrancis}; we can use Cauchy-Schwartz to write:
\begin{align}
	| \overline \phi_{\pi,t}^\top\( \widehat \theta_t - \mathring \theta_t(\Vhat_{t+1})\) | \leq \| \overline \phi_{\pi,t} \|_{\Sigma^{-1}_{t}} \|  \widehat \theta_t - \mathring \theta_t(\Vhat_{t+1}) \|_{\Sigma_{t}} \leq 
	\sqrt{\alpha_t}\| \overline \phi_{\pi,t} \|_{\Sigma^{-1}_{t}}
\end{align}
Using  \fullref{def:InherentBellmanError} we can write:
\begin{align}
	| \overline \phi_{\pi,t}^\top \mathring \theta_t(\Vhat_{t+1}) - \E_{x_t\sim\pi}\T^P_t \Vhat_{t+1}(x_t,\pi_t(x_t)) | \leq \IBEE(\mathcal Q_{t},\mathcal Q_{t+1}).
\end{align}
Combining the two expression gives:
\begin{align}
	& | \E_{x_t \sim \pi}  \Qhat_t(x_t,\pi_t(x_t)) - \E_{x_{t+1} \sim \pi}\Vhat_{t+1}(x_{t+1})|  \\
	& = | \E_{x_t \sim \pi} \big[ \Qhat_t(x_t,\pi_t(x_t)) - \T^P_t \Vhat_{t+1}(x_t,\pi_t(x_t)) \big] | \\
	&  = | \overline \phi_{\pi,t}^\top \widehat \theta_t - \E_{x_t \sim \pi}  \T^P_t (\Vhat_{t+1})(x_t,\pi_t(x_t)) | \\
	& = | \overline \phi_{\pi,t}^\top \widehat \theta_t - \overline \phi_{\pi,t}^\top \mathring \theta_t(\Vhat_{t+1}) + \overline \phi_{\pi,t}^\top \mathring \theta_t(\Vhat_{t+1}) -  \E_{x_t \sim \pi} \T^P_t (\Vhat_{t+1})(x_t,\pi_t(x_t)) | \\
		& \leq | \overline \phi_{\pi,t}^\top \widehat \theta_t - \overline \phi_{\pi,t}^\top \mathring \theta_t(\Vhat_{t+1})| + |\overline \phi_{\pi,t}^\top \mathring \theta_t(\Vhat_{t+1}) -  \E_{x_t \sim \pi} \T^P_t (\Vhat_{t+1})(x_t,\pi_t(x_t)) | \\
		& \leq \sqrt{\alpha_t}\| \overline \phi_{\pi,t} \|_{\Sigma^{-1}_{t}} + \IBEE(\mathcal Q_t,\mathcal Q_{t+1}).
\end{align}
To show the upper bound if $\pi$ is the greedy policy with respect to $\Qhat$ then we can equivalently write $\Vhat_t(x_t) = \Qhat_t(x_t,\pi_t(x_t))$
\begin{align}
	| \E_{x_t \sim \pi} \Vhat_t(x_t) - \E_{x_{t+1} \sim \pi}  \Vhat_{t+1}(x_{t+1}) | \leq \sqrt{\alpha_t}\| \overline \phi_{\pi,t} \|_{\Sigma^{-1}_{t}} + \IBEE(\mathcal Q_t,\mathcal Q_{t+1}).
\end{align}
Induction now shows the upper bound.

To show the lower bound, for a generic policy $\Vhat_t(x_t) \geq \Qhat_t(x_t,\pi_t(x_t))$
and so
\begin{align}
	 \E_{x_{t} \sim \pi}  \Qhat_{t}(x_{t},\pi_{t}(x_{t}))  &  \geq -\sqrt{\alpha_t}\| \overline \phi_{\pi,t} \|_{\Sigma^{-1}_{t}} - \IBEE(\mathcal Q_t,\mathcal Q_{t+1}) + \E_{x_{t+1} \sim \pi}  \Vhat_{t+1}(x_{t+1}) \\
	 &  \geq -\sqrt{\alpha_t}\| \overline \phi_{\pi,t} \|_{\Sigma^{-1}_{t}} - \IBEE(\mathcal Q_t,\mathcal Q_{t+1}) + \Qhat_{t+1}(x_{t+1},\pi_{t+1}(x_{t+1})).
\end{align}
Induction concludes.
\end{proof}

\newpage
\begin{proposition}[Batch \lsvi{} Guarantees (\cref{alg:lsviBatch})]
	\label{prop:BatchLSVI}
	Under the good event of \fullref{def:GoodEventLSVI} assume that \begin{align}
\forall t \in [H] \quad \quad \| \theta^r_t\|_2 \leq \frac{R}{H}
\end{align}
If $\Vhat$ and $\pihatstar$ are the value function and policy returned by \cref{alg:lsviBatch} then 
\begin{align*}
	\E_{x_1 \sim \rho}  \( \Vstar_1 - \Vhat_1\)(x_1)  & \leq \sum_{t=1}^H \Bigg[ 2E_t + R\( \IBEE(\mathcal Q_t,\mathcal Q_{t+1}) + \sqrt{\alpha_t}\| \overline \phi_{\pistar,t} \|_{\Sigma^{-1}_t}\) \Bigg] \\
 \E_{x_1 \sim \rho }\( \Vhat_{1} - V_1^{\pihatstar}\)(x_1)  & \leq \sum_{t=1}^H \Bigg[ 2E_t +R\(\IBEE(\mathcal Q_t,\mathcal Q_{t+1}) + \sqrt{\alpha_t}\| \overline \phi_{\pihatstar,t} \|_{\Sigma^{-1}_t}\)\Bigg].
 \numberthis{\label{eqn:FinalSplit}}
\end{align*}
\end{proposition}

\begin{proof}
Boundness of the iterates $\| \widehat \theta^r + \widehat \theta^{R+PV}\|_2$ is ensured by \fullref{lem:IntermediateBoundnes}.
Consider a generic timestep $t$; using the Bellman equation and the fact that $\Vhat_t(x_t) \geq \Qhat_t(x_t,\pistar_t(x_t))$ gives
\begin{align}
	\E_{x_t \sim \pistar}  \( \Vstar_t - \Vhat_t\)(x_t) & \leq \E_{x_t \sim \pistar} r_t(x_t,\pistar_t(x_t)) + \E_{x_{t+1} \sim \pistar}  \Vstar_{t+1}(x_{t+1}) - \E_{x_t \sim \pistar} \phi_{t}(x_t,\pistar_t(x_t))^\top \(\widehat \theta_t^r + \widehat \theta_t \)  \\
	& \leq E_t + \overline \phi_{\pistar,t}^\top\theta^r_t + \E_{x_{t+1} \sim \pistar}  \Vstar_{t+1}(x_{t+1}) - \E_{x_t \sim \pistar} \phi_{\pistar,t}^\top \(\widehat \theta_t^r + \widehat \theta_t \)
	\end{align}
	Next, under the good event of \cref{def:GoodEventLSVI} we can write:
	\begin{align}
& \leq 2E_t + \overline \phi_{\pistar,t}^\top\theta^r_t + \E_{x_{t+1} \sim \pistar}  \Vstar_{t+1}(x_{t+1}) -  \overline \phi_{\pistar,t}^\top \theta_t^r  \\
& -  \E_{x_{t+1} \sim \pistar}  \Vhat_{t+1}(x_{t+1}) + R[\IBEE(\mathcal Q_t,\mathcal Q_{t+1}) + \sqrt{\alpha_t} \| \overline \phi_{\pistar,t} \|_{\Sigma^{-1}_t}]
	\end{align}
	Induction gives the first statement.
	
	Now again we start with the definition of expected feature and the Bellman equation:
	\begin{align}
	& \E_{x_t \sim \pihatstar }\( \Vhat_{t} - V_t^{\pihatstar}\)(x_t)  = \overline \phi_{\pihatstar,t}^\top (\widehat \theta^r_t + \widehat \theta_t) - \E_{x_t \sim \pihatstar} r_t(x_t,\pihatstar_t(x_t)) - \E_{x_{t+1} \sim \pihatstar}  V_{t+1}^{\pihatstar}(x_{t+1}) \\
	 & \leq \overline \phi_{\pihatstar,t}^\top \theta^r + E_t +  R[\IBEE(\mathcal Q_t,\mathcal Q_{t+1}) + \sqrt{\alpha_t}\| \overline \phi_{\pihatstar,t} \|_{\Sigma^{-1}_t}] + \\
	 & + \E_{x_{t+1}\sim \pihatstar}\Vhat_{t+1}(x_{t+1})  - \overline\phi_{\pihatstar,t}^\top \theta^r_t + E_t +  \E_{x_{t+1} \sim \pihatstar} V_{t+1}^{\pihatstar}(x_{t+1}).
	 \end{align}
	 Induction again concludes.
\end{proof}

\newpage
\section{Design of Experiments}
\label{sec:DoE}

We show that obtaining $\| \overline \phi_{\pi,t} \|_{\Sigma_t^{-1}} \leq \frac{\epsilon}{H\sqrt{\alpha_t}} = \epsilon'$ suffices; we assume $\IBEE(\mathcal Q_t, \mathcal Q_{t+1}) = E_t = 0$ for simplicity as well as $d_1 = \dots = d_H$. We immediately have that
\begin{align}
	\sum_{t=1}^H \sqrt{\alpha_t} \| \overline \phi_{\pi,t} \|_{\Sigma_t^{-1}} \leq H \times \sqrt{\alpha_t} \times \frac{\epsilon}{H\sqrt{\alpha_t}} = \epsilon.
\end{align}
Thus, summing the two equations in  \cref{eqn:FinalSplit} for any linear reward function with $\|\theta_t \|_2 \leq \frac{1}{H}$ ensures an $\epsilon$-optimal policy on that reward function is returned.

The Kiefer-Wolfowitz theorem in \cite{lattimore2020bandit} guarantees such reduction in $\| \overline \phi_{\pi,t}\|_{\Sigma_t^{-1}}$ using $\widetilde O(d^2 + \frac{d}{(\epsilon')^2}) = \widetilde O(d^2 + \frac{dH^2\alpha_t}{\epsilon^2})$ samples at every level / timestep if $G$-optimal design is used. After sampling all levels and substituting the value for $\alpha_t$ in table $\ref{tab:MainNotation}$ the sample complexity of doing G-optimal design becomes $ \widetilde O(d^2 + \frac{d^2H^3}{\epsilon^2})$. 

Notice that this setting can model MDPs with rewards in $[0,1/H]$ and value functions in $[0,1]$; moving to the standard setting with rewards in $[0,1]$ and value function in $[0,H]$ adds $H^2$ to the sample complexity to obtain an $\epsilon$-optimal policy.
\black

\newpage
\section{Analysis of \Alg{}}
\subsection{Generating Bounded Iterates}
\label{sec:GeneratingBoundedIterates}

The following lemma ensures \Alg{} generates bounded iterates for an appropriate choice of $\sigma$.
\begin{lemma}[Boundness at Exploratory Timestep]
\label{lem:ExploratoryBoundness}
In episode $k$ of phase $p$, if 
\begin{align}
	\lambda_{min}(\Sigma_{pk}) & \geq \MagicValue{p}\sigma \\
	\xi_p & \sim \mathcal N(0,\sigma\Sigma^{-1}_{pk})
\end{align}	
then
\begin{align}
 	\| \xi_{p} \|_2 \leq \frac{1}{2} 
\end{align}
on the good event of \fullref{def:GoodEventFrancis}.
\end{lemma}
\begin{proof}
Directly by the choice of $\sigma$ and the definition of good event for \Alg{} (see \fullref{def:GoodEventFrancis}).
\end{proof}

\subsection{Derandomization}
\label{sec:Derandomization}

The following lemma relates the sampling of the algorithm to a procedure that selects the policy / parameter leading to the area of highest (scaled) uncertainty.

\begin{lemma}[Derandomization]
\label{lem:Derandomization}

Outside of the failure event, assume that for any policy $\pi$,
\begin{align}
\sum_{t=1}^{p-1} \Big[\IBEE(\mathcal Q_{t},\mathcal Q_{t+1}) + \sqrt{\alpha_{t}}\|\overline \phi_{\pi,t}\|_{\Sigma_{t}^{-1}} \Big] \leq \overline \epsilon
\end{align}
for some scalar $\overline \epsilon > 0$.
Consider sampling
\begin{align}
\label{eqn:OneHundred}
	\xi_p \sim \mathcal N(0,\sigma \Sigma_{pk}^{-1}),
\end{align}
define $\textsc{r}_p(s,a) = \phi_p(s,a)^\top\xi_p$ and let $\Vhat$ be the value function computed by \lsvi{}$(p, \textsc{r}_p\mathcal D)$ (see \cref{alg:lsviFrancis}). Then for a fixed constant $q \in \R$
\begin{align}
	 \Pro\(\E_{x_1 \sim \rho}\Vhat_{1}(x_1) - \overline\epsilon  > \Mdef{\sigma}{\Sigma_{pk}}{p} \) \geq q.
\end{align}
if 
\begin{align}
	\Mdef{\sigma}{\Sigma_{pk}}{p} & \geq \overline \epsilon \\
	\| \xi_p \|_2 & \leq \frac{1}{2}.
\end{align}
\end{lemma}

\begin{proof}
Define the maximizer of the ``scaled uncertainty'' in a generic episode $k$ of phase $p$:
\begin{align}
\(\overset{\triangle}{\pi},\overset{\triangle}{\eta}\) \defeq \argmax_{\substack{ \pi \\ \| \eta \|_{\Sigma_{pk}} \leq \sqrt{\sigma}}} |\overline \phi_{\overset{\triangle}{\pi},p}^\top \eta | \end{align}
as the policy / parameter that maximizes the uncertainty.

Next, let $\overline \pi$ be the policy selected by the agent, through \lsvi{}, corresponding to the sampled parameter $\xi_p$ and let $\Qhat,\Vhat$ be the (action) value functions.
Since $\overline \pi$ is the maximizing policy for $\Qhat$, we must have:
\begin{align}
\E_{x_1 \sim \rho}\Vhat_{1}(x_1) = \E_{x_1 \sim \rho}\Qhat_{1}(x_1,\overline \pi_1(x_1)) \geq  \E_{x_1 \sim \rho}\Qhat_{1}(x_1,\overset{\triangle}{\pi}_1(x_1)).
\end{align}
In addition on the good event for \lsvi{} \fullref{lem:Telescope} gives: 
\begin{align}
	\E_{x_1 \sim \rho}\Vhat_{1}(x_1) \geq \E_{x_1 \sim \rho}\Qhat_{1}(x_1,\overset{\triangle}{\pi}_1(x_1)) & \geq \sum_{t=1}^{p-1} \Big[ - \IBEE(\mathcal Q_{t},\mathcal Q_{t+1}) - \sqrt{\alpha_{t}}\|\overline \phi_{\overset{\triangle}{\pi},t}\|_{\Sigma_{t}^{-1}} \Big] + \underbrace{\E_{x_{p} \sim \overset{\triangle}{\pi}}\Qhat_{p}(x_{p},\overset{\triangle}{\pi}_{p}(x_{p}))}_{(\overline \phi_{\overset{\triangle}{\pi},p})^\top \xi_p}.
\end{align}
Subtracting $\overline \epsilon$ to both sides and using the hypothesis gives
\begin{align*}
 \E_{x_1 \sim \rho}\Vhat_{1}(x_1) - \overline\epsilon   & \geq -2\overline \epsilon + (\overline \phi_{\overset{\triangle}{\pi},p})^\top \xi_p.
\numberthis{\label{eqn:107}}
\end{align*}

We can now proceed to bound the quantity of interest:
\begin{align}
	&  \Pro\( \E_{x_1 \sim \rho}  \Vhat_{1}(x_1)  - \overline\epsilon \geq (\overline \phi_{\overset{\triangle}{\pi},p})^\top \overset{\triangle}{\eta}\) \\
	 \geq & \Pro \( -2\overline \epsilon + \overline \phi_{\overset{\triangle}{\pi},p}^\top \xi_p  \geq (\overline \phi_{\overset{\triangle}{\pi},p})^\top \overset{\triangle}{\eta} \) \\
	 = & \Pro \( \overline \phi_{\overset{\triangle}{\pi},p}^\top \xi_p \geq  \underbrace{\vphantom{\sum_{t=1}^{p-1}} 2\overline\epsilon}_{\text{Error in Propagating the Uncertainty}} + \underbrace{ \vphantom{\sum_{t=1}^{p-1}} (\overline \phi_{\overset{\triangle}{\pi},p})^\top \overset{\triangle}{\eta}}_{\text{Uncertainty in the Level to Learn}} \) \geq q
\end{align}
Notice that  $\xi_p$ is independent of $\overline \phi_{\overset{\triangle}{\pi}}$ when conditioned on the $\Sigma_{tk}$. 
The last step is an application of \fullref{lem:UncertantyOverestimation}  as long as the condition
\begin{align}
	\overline \epsilon & \leq \max_{\phi,\| \eta \|_{\Sigma_{pk}} \leq \sqrt{\sigma}}   \overline \phi_{\pi,p}^\top\eta 
	\end{align}
is met.
\end{proof}

\newpage
\begin{lemma}[Uncertainty Overestimation]
\label{lem:UncertantyOverestimation}
Let $\overline \epsilon, \sigma$ be positive scalars, and let $\Sigma$ be an spd matrix and let
\begin{align}
\xi  \sim \mathcal N(0,\sigma\Sigma^{-1})	
\end{align}
be the associated random vectors. For a fixed vector $\phi$ we have that 
\begin{align}
\label{eqn:162}
	\Pro\(\phi^\top\xi \geq \max_{\phi,\| \eta \|_{\Sigma} \leq \sqrt{\sigma}}   \phi^\top\eta + 2\overline \epsilon \) \geq \Phi(-3) \defeq q
\end{align}
where $\Phi(\cdot)$ is the normal CDF function as long as the condition
\begin{align}
	\overline \epsilon & \leq \max_{\phi,\| \eta \|_{\Sigma} \leq \sqrt{\sigma}}   \phi^\top\eta  = \sqrt{\sigma}\|\phi\|_{\Sigma^{-1}}
\end{align}
holds true.
\end{lemma}
\begin{proof}
Before we prove the statement, we notice that the equivalent expression $\max_{\phi,\| \eta \|_{\Sigma} \leq \sqrt{\sigma}}   \phi^\top\eta  = \sqrt{\sigma}\|\phi\|_{\Sigma^{-1}}$ can be found in chapter 19 of \citep{lattimore2020bandit} about the \textsc{LinUCB} algorithm, see also \fullref{lem:LinearBanditBonus}.	
For any fixed $\Sigma$, we have that $\xi \sim \mathcal N(0,\sigma\Sigma^{-1})$ is independent of $\phi$ by hypothesis, and so the inner product below is normally distributed
\begin{align}	
\phi^\top \xi \sim \mathcal N\(0,\sigma\phi^\top \Sigma^{-1} \phi\),
\end{align}
or equivalently
\begin{align}
\phi^\top \xi \sim \mathcal N\(0,\sigma\| \phi  \|^2_{\Sigma^{-1}}\).
\end{align}
Rescaling by its standard deviation leads to the following definition:
\begin{align}
X \defeq \frac{\phi^\top \xi }{\sqrt{\sigma} \| \phi  \|_{\Sigma^{-1}}} \sim \mathcal N\(0,1\).
\end{align}
The step below follows
\begin{align}
\Pro\( \phi^\top \xi \geq \sqrt{\sigma} \| \phi \|_{\Sigma^{-1}} + 2\overline \epsilon \) & = \Pro\(X \geq 
	1 + \frac{2\overline \epsilon}{\sqrt{\sigma}\| \phi  \|_{\Sigma^{-1}}} \).
	\end{align}
The rhs above is $\geq \Phi(-3)$ as long as
\begin{align}
	\overline \epsilon \leq \sqrt{\sigma}\|\phi\|_{\Sigma^{-1}}.
\end{align}
The thesis follows from the definition of the normal CDF.
\end{proof}

\newpage
\subsection{Learning an Epoch}
\label{sec:LearningEpoch}

The following lemma is key to our analysis and shows the number of episodes required to reduce the scaled uncertainty to the minimum allowable ($\approx \overline \epsilon > 0$). In an epoch the value for $\sigma$ is fixed.

\begin{lemma}[Learning an Epoch]
\label{lem:LearningEpoch}
Let $\underline k$ and $\overline k$ be the starting and ending episodes in epoch $e$ of phase $p$. If the following statements hold:
\begin{enumerate}
	\item for any policy $\pi$ it holds that
 $\sum_{t=1}^{p-1} \Big[ \IBEE(\mathcal Q_{t},\mathcal Q_{t+1}) + \sqrt{\alpha_{t}} \| \overline \phi_{\pi,t} \|_{\Sigma_{t}^{-1}} \Big] \leq \overline \epsilon$
	\item $	\lambda_{min}(\Sigma_{p\underline k})  \geq \MagicValue{p}{\sigma}$ (this ensures boundness of $\| \xi_p\|_2$ in \fullref{lem:Derandomization})
	\item $\frac{\Lphi^2}{\lambda} \leq 1$ (always satisfied by our choice $\Lphi = 1$ and $\lambda = 1$)
	\item $\lambda > 1$ (always satisfied by our choice $\lambda = 1$)
\end{enumerate}
then after at most
\begin{align}
	k_{max} = \overline k - \underline k = \Bigg\lceil\frac{2}{1-q} \times \frac{(\sqrt{\gamma(\rho)D_p}+A)^2}{(\epsilon'')^2} \Bigg\rceil
\end{align}
episodes we must have
\begin{align}
\label{eqn:110}
	\Mdef{\sigma}{\Sigma_{p\overline k}}{p} \leq \epsilon''
\end{align}
on the good event \fullref{def:GoodEventFrancis} provided that
\begin{align}
\epsilon'' \geq \overline \epsilon.
\end{align}
\end{lemma}
\begin{proof}

First notice that if the eigenvalue condition is satisfied for at a given episode $\underline k$ then it must be satisfied for all successive episodes $k \geq \underline k$ since $\Sigma_{tk} \succeq \Sigma_{t\underline k}$.
In particular, define the events
\begin{align}
		\mathcal C_k & \defeq \Ckdef \\
		\mathcal E_k & \defeq \Ekdef.
\end{align}

We examine what happens in those episodes where $\mathcal E_k$ occurs (notice that $\Pro(\mathcal E_k \mid \mathcal C_k) \geq q$ thanks to \fullref{lem:Derandomization}).  

Let $k(e,i)$ be the $i$-th consecutive episode index in epoch $e$ of phase $p$ such that $\mathcal E_{k(e,i)}$ occurs (so in $k(e,1),k(e,2),\dots$ we have that $\mathcal E_{k(e,1)}, \mathcal E_{k(e,2)}$ occurs). Since $\| \xi_{pk(e,i)}\|_2 \leq 1/2$ in the good event of \fullref{def:GoodEventFrancis}, we can use  \fullref{lem:ExploratoryBoundness} and \fullref{lem:Telescope} to write
\begin{align}
		\E_{x_1 \sim \rho}\Vhat_{pk(e,i),1}(x_1) - \overline \epsilon  \leq   \phi_{pk(e,i)}^\top\xi_{pk(e,i)} + \zeta_{pk(e,i)}.
\end{align}
where 
\begin{align}
	\zeta_{pk(e,i)} \defeq \zetadef
\end{align}
Let $i_{max}$ be a fixed positive constant to be determined later. Taking average of the previous display up to $i_{max}$ gives:
\begin{align}
\label{eqn:168}
		& \frac{1}{i_{max}} \sum_{i=1}^{i_{max}} \E_{x_1 \sim \rho}\Vhat_{pk(e,i),1}(x_1) - \overline \epsilon  \leq  \frac{1}{i_{max}} \sum_{i=1}^{i_{max}} \( \phi_{pk(e,i)}^\top\xi_{pk(e,i)} + \zeta_{pk(e,i)} \).
\end{align}

Under the good event of \fullref{def:GoodEventFrancis} we have 
\begin{align}
		\frac{1}{i_{max}}\sum_{i=1}^{i_{max}} \zeta_{pk(e,i)} & \leq \frac{A}{\sqrt{i_{max}}} \end{align}
		with $A = \widetilde O (1)$. For the remaining term, 
using Cauchy-Schwartz, and the fact that we are on the good event (see \fullref{def:GoodEventFrancis}) gives
\begin{align}
 \frac{1}{i_{max}}\sum_{i=1}^{i_{max}}  \phi_{pk(e,i)}^\top\xi_{pk(e,i)} & \leq \frac{1}{i_{max}}\sum_{i=1}^{i_{max}}  \|\phi_{pk(e,i)}\|_{\Sigma^{-1}_{pk(e,i)}} \underbrace{\| \xi_{pk(e,i)}\|_{\Sigma_{pk(e,i)}}}_{\sqrt{\gamma_t(\sigma)}}
\end{align}
After one more Cauchy-Schwartz we obtain the upper bound below:
 \begin{align}
 &  \leq  \frac{\sqrt{\gamma_t(\sigma)}}{i_{max}}\sum_{i=1}^{i_{max}} \|\phi_{pk(e,i)}\|_{\Sigma^{-1}_{pk(e,i)}}  \leq  \sqrt{\frac{\gamma_t(\sigma)}{i_{max}} \sum_{i=1}^{i_{max}} \|\phi_{pk(e,i)}\|^2_{\Sigma^{-1}_{pk(e,i)}}}.
 \end{align}
 We focus on the sum of squared features; by \fullref{lem:MaxEvaluBound} and the lemma's hypothesis 
 \begin{align}
 \|\phi_{pk(e,i)}\|^2_{\Sigma^{-1}_{pk(e,i)}} \leq \frac{1}{\lambda}\|\phi_{pk(e,i)}\|^2_2 \leq \frac{\Lphi^2}{\lambda} \leq 1
 \end{align}
 and so the sum of squared features becomes\footnote{notice that we are not accounting for the the progress made in episodes where $\mathcal E_k$ does not occur} (using the elliptic potential lemma, see lemma 11 in \citep{Abbasi11}):
 \begin{align}
 	\sum_{i=1}^{i_{max}} \|\phi_{pk(e,i)}\|^2_{\Sigma^{-1}_{pk(e,i)}} =  
 	\sum_{i=1}^{i_{max}} \min \{1 , \|\phi_{pk(e,i)}\|^2_{\Sigma^{-1}_{pk(e,i)}} \} \leq \ln\( \frac{\det\Sigma_{pk(e,i_{max})}}{\det\Sigma_{p,\underline k}} \) \leq \ln \det\Sigma_{pk(e,i_{max})}.
 \end{align}
 The last step follows because $\Sigma_{p\underline k} \succeq \lambda I \succeq I$, an so $\det(\Sigma_{p\underline k}) \geq \det I = 1$. Let $D_{p} = d_p \ln(1+k\Lphi^2/d) = \widetilde O(d_p)$ be an upper bound to $\ln \det\Sigma_{pk(e,i_{max})}$ (see lemma 10 in \cite{Abbasi11}). We can claim that an upper bound to \cref{eqn:168} is
 \begin{align}
 	\leq \frac{A + \sqrt{\gamma_t(\sigma)D_p}}{\sqrt{i_{max}}}.
 \end{align}
Since we're summing over episode indexes where $\mathcal E_{k(e,i)}$ holds, it follows that
 \begin{align}
 \label{eqn:SumEq}
 	 \frac{1}{i_{max}}\sum_{i=1}^{i_{max}} \Bigg[ \Mdef{\sigma}{\Sigma_{pk(e,i)}}{p} \Bigg] \leq  \frac{A + \sqrt{\gamma(\sigma)D_p}}{\sqrt{i_{max}}}
 \end{align}
if each term in the summation in the lhs is $\geq \epsilon''$ (the condition is needed to apply \fullref{lem:Derandomization}; if it does not hold the lemma's thesis is satisfied). By \fullref{lem:UncertaintyLemma}
 \begin{align}
 	\Mdef{\sigma}{\Sigma_{p,k(e,i+1)}}{p} \leq \Mdef{\sigma}{\Sigma_{p,k(e,i)}}{p} 
 \end{align}
Since the terms in the lhs of \cref{eqn:SumEq} are strictly decreasing, the last one must be smaller than the average, which implies we must obtain 
\begin{align}
	\Mdef{\sigma}{\Sigma_{pk(e,i_{max})}}{p} \leq \epsilon''
\end{align}
after
\begin{align}
	i_{max} \geq \frac{(\sqrt{\gamma_t(\rho)D_p}+A)^2}{(\epsilon'')^2}
\end{align}
episodes
provided that\footnote{This condition is recurrent in this proof, and is used to invoke \fullref{lem:Derandomization}, but if it doesn't hold the thesis is automatically satisfied.}
\begin{align}
\epsilon'' \geq \overline \epsilon.
\end{align}
We can finally compute how big $k_{max}$ (the total number of episodes in the epoch) needs to be: from \fullref{def:GoodEventFrancis} if 
\begin{align}
\label{eqn:lmaxMin}
		k_{max} \geq \frac{1}{4} \times \frac{2\ln(\frac{1}{\delta''})}{1-q}
\end{align}
then we can write
\begin{align}
		\frac{i_{max}}{k_{max}} \geq \frac{1-q}{2}.
\end{align}
(recall $i_{max}$ is the the number of episodes where $\mathcal E_k$ occurs: $ i_{max} = \sum_{k=1}^{k_{max}} \1\{\mathcal{E}_k \mid \mathcal C_k \}$).
Therefore, a total number of episodes
\begin{align}
	k_{max} = \Bigg\lceil\frac{2}{1-q} \times \frac{(\sqrt{\gamma_t(\rho)D_p}+A)^2}{(\epsilon'')^2} \Bigg\rceil
\end{align}
suffices (as this automatically satisfies \cref{eqn:lmaxMin}).
\end{proof}

\begin{lemma}[Uncertainty Lemma]
\label{lem:UncertaintyLemma}
Let $\overline k$ and $k$ be two generic episodes in an epoch $e$ in phase $p$ such that $\overline k \geq k$. We have that
\begin{align}
		\Mdef{\sigma}{\Sigma_{p\overline k}}{p} \leq 	\Mdef{\sigma}{\Sigma_{pk}}{p}.
\end{align}
In addition, for positive real numbers $\rho_1 \leq \rho_2$ and a generic spd matrix $\Sigma$ we also have
\begin{align}
\label{eqn:bothpr}
		  \Mdef{\rho_1}{\Sigma}{p} = \sqrt{\frac{\rho_1}{\rho_2}} \Mdef{\rho_2}{\Sigma}{p}.
\end{align}
\end{lemma}

\begin{proof}
Since $\Sigma_{p\overline k} \succeq \Sigma_{pk}$ (this notation means $\Sigma_{p\overline k}$ is more positive definite than $\Sigma_{pk}$, more precisely $\phi^\top \Sigma_{p\overline k} \phi \geq \phi^\top \Sigma_{pk} \phi$ for all $\phi$)
we have the set inclusion
\begin{align}
	\{ \eta \mid \| \eta \|_{\Sigma_{p\overline k}} \leq  \sqrt{\sigma} \} \subseteq \{ \eta \mid \| \eta \|_{\Sigma_{pk}} \leq  \sqrt{\sigma}	\} 
\end{align}
Since we're maximizing over a smaller set, the first result follows.

For the second statement, recall we can rewrite the programs in \cref{eqn:bothpr} (see chapter 19 of \citep{lattimore2020bandit} about \textsc{LinUCB} or equivalently \fullref{lem:LinearBanditBonus} ); here we identify the feature of an action in \textsc{LinUCB} with $\overline \phi_{\pi,p}$) as
\begin{align}
\max_{\pi}\sqrt{\rho_1}\|\overline \phi_{\pi,p} \|_{\Sigma^{-1}}
\end{align}
for the lhs and
\begin{align}
 \max_{\pi}\sqrt{\frac{\rho_1}{\rho_2}} \sqrt{\rho_2}\|\overline \phi_{\pi,p} \|_{\Sigma^{-1}}
\end{align}
for the rhs, showing equality.
\end{proof}

\newpage
\subsection{Learning a Phase}
\label{sec:LearningALevel}

In this section we show how \Alg{} learns a phase (i.e., the \emph{dynamics} at a certain \emph{timestep}) and compute the total number of episodes required to do so. This is where the \emph{explorability} condition is used.
\begin{lemma}[Learning a Level]
\label{lem:LearningALevel}
Consider phase $p$ and let the following hypotheses hold
\begin{enumerate}
	\item $\sum_{t=1}^{p-1} \Big[ \IBEE(\mathcal Q_{t},\mathcal Q_{t+1}) + \sqrt{\alpha_{t}} \| \overline \phi_{\pi,t} \|_{\Sigma_{t}^{-1}} \Big] \leq \overline \epsilon$ \\
	\item $\(\frac{\nu}{ \epsilon}\)^2 \geq 2 \times \MagicValue{p}$ 
\end{enumerate}
Then after at most ($e_{max} = \widetilde O(1)$ and $\sigma_{e_{max}}$ are defined in the proof)
\begin{align}
	n(t) = \Bigg\lceil\frac{2}{1-q} \times \frac{(\sqrt{\gamma_t(\sigma_{e_{max}})D_p}+A)^2}{ \epsilon^2} \Bigg\rceil \times e_{max} = \widetilde O\( \frac{d_p^2H^2\alpha_p}{ \epsilon^2} \) =  \widetilde O\(\frac{d^2_p\times H^2 (d_p + d_{p+1})}{\epsilon^2}\) 
\end{align}
episodes it must hold that
\begin{align}
	\Mdef{\alpha_p}{\Sigma_{p\overline k}}{p} \leq \frac{\epsilon}{2H}
\end{align}
\end{lemma}
\begin{proof}
Let $\sigma_1,\sigma_2,\dots$ be the sequences of the $\sigma$ parameter chosen in the different epochs, and additionally 
	\begin{align}
  \sigma_{Start} =  1/\( \MagicValue{p}\).
	\end{align}
We proceed by induction, with the following inductive hypothesis:
\begin{inductivehypothesis}
\label{hypo:LL}
	In phase $p$ the following conditions hold
	\begin{enumerate}[label=(\alph*)]
		\item $\lambda_{min}(\Sigma_{pk(e,1)})  \geq \MagicValue{p}{\sigma_e}$ \quad (at the beginning of epoch $e$)
		\item $\sigma_{e} = 2^{e-1}\sigma_{Start}$ \quad (at the beginning of epoch $e$)
	\end{enumerate}
\end{inductivehypothesis}
To show that the inductive hypothesis is satisfied in the base case ($e=1$), notice that $(b)$ holds by definition and $(a)$ holds by setting $\lambda = 1$. Now we show the inductive step.

Since the inductive hypothesis satisfies the hypothesis of   \fullref{lem:LearningEpoch}, on the good event \fullref{def:GoodEventFrancis} it immediately follows that
\begin{align}
\label{eqn:139}
		\Mdef{\sigma_e}{\Sigma_{pk}}{p} \leq  \epsilon''
\end{align}
after $k_{max}$ episodes (see \fullref{lem:LearningEpoch}). Here in particular $k$ is the last episode of epoch $e$. The explorability condition in \fullref{main.def:Explorability} implies that
\begin{align}
	\forall \eta\neq 0, \; \exists \pi \quad \text{such that} \quad \overline \phi_{\pi,t}^\top \frac{\eta}{\| \eta \|_2} \geq \nu_{min}.
\end{align}

Consider the normalized evector $v$ corresponding to the minimum eigenvalue $q > 0$ for $\Sigma_{pk}$ and define:
\begin{align}
	\eta = q v.
\end{align}
We're interested in determining the maximum $q$ so that the constraint in the program \cref{eqn:139} is still satisfied, i.e., the condition below
\begin{align}
 \sigma_e \geq \| q v \|^2_{\Sigma_{pk}} = \(q v \)^\top \Sigma_{pk} \( q v \) = q^2 \lambda_{min} \( \Sigma_{pk} \)
\end{align}
gives the maximum value for $q$
\begin{align}
	q = \sqrt{\frac{\sigma_e}{\lambda_{min}(\Sigma_{pk})}}
\end{align}
in order for $qv$ to satisfy $\| qv \|_{\Sigma_{pk}} \leq \sqrt{\sigma_e}$. In other words, the $qv$ vector so defined is a feasible solution to the first program below, justifying one inequality:
\begin{align}
\epsilon'' &\geq \max_{\pi,\| \eta \|_{\Sigma_{pk}} \leq \sqrt{\sigma_e}} \big[ \overline \phi_{\pi,t}^\top \eta \big] \geq \max_{\pi} \big[ \overline \phi_{\pi,t}^\top (qv) \big] = \| qv \|_2 \max_{\pi} \( \overline \phi_{\pi,t}^\top \frac{ (qv)}{\| qv \|_2} \)\\ &\geq \| qv \|_2 \nu_{min} = q\nu_{min} = \sqrt{\frac{\sigma_e}{\lambda_{min}(\Sigma_{pk})}} \nu_{min}.
\end{align}
Solving for $\lambda_{min}$ gives:
\begin{align}
\lambda_{min}(\Sigma_{pk})\geq \sigma_e \(\frac{\nu_{min}}{\epsilon}\)^2 \geq \sigma_e \times 2 \times \MagicValue{p} = \sigma_{e+1} \times \MagicValue{p}
	\end{align}
	Therefore the inductive hypothesis must hold for $e+1$ as well, in other words, the statement in inductive hypothesis \ref{hypo:LL} must hold for all $e$.

Now we determine the required value for $\rho$ at the end of the phase. We want to ensure 
\begin{align}
			\Mdef{\alpha_p}{\Sigma_{pk}}{p} \leq \frac{ \epsilon}{2H}
\end{align}
where now $k$ \emph{is the episode at the end of phase $p$}.
Since the inductive hypothesis holds in epoch $e$, \cref{lem:LearningEpoch} ensures
\begin{align}
		 \Mdef{\sigma}{\Sigma_{pk}}{p} \leq \epsilon''.
\end{align}
We combine the above finding with a scaling argument given by \fullref{lem:UncertaintyLemma} that gives:
\begin{align}
			 \Mdef{\alpha_p}{\Sigma_{pk}}{p} = \sqrt{\frac{\alpha_p}{\sigma_e}} \times \(\Mdef{\sigma_e}{\Sigma_{pk}}{p}  \)  \leq \sqrt{\frac{\alpha_p}{\sigma_e}}  \epsilon''.
\end{align}
Requiring the above rhs to be $\leq \frac{\epsilon}{2H}$ gives a condition on the number of epochs $e_{max}$ required ($e_{max}$ is the number of epochs) and on $\sigma_{e_{max}}$; setting $\epsilon'' = \epsilon$ gives
\begin{align}
	\sqrt{\frac{\alpha_p}{\sigma_{e_{max}}}}  \epsilon \leq \frac{\epsilon}{2H} & \rightarrow \sqrt{\frac{\sigma_{e_{max}}}{\alpha_p}} \geq 2H \\
	& \rightarrow \sigma_{e_{max}} = 2^{{e_{max}}-1}\sigma_{Start} \geq 4H^2\alpha_p \quad \quad \text{(by induction)}\\
	& \rightarrow 2^{{e_{max}}-1} \geq \frac{4H^2\alpha_p}{\sigma_{Start}} \rightarrow e_{max} = \Bigg\lceil 1 + \ln_2\( \frac{4H^2\alpha_p}{\sigma_{Start}} \) \Bigg\rceil.
\end{align}
In every epoch, $\epsilon'' = \epsilon$ and so the number of episodes necessary to achieve the required precision is (see \fullref{lem:LearningEpoch}):
\begin{align}
	\sum_{e=1}^{e_{max}} \Bigg\lceil\frac{2}{1-q} \times \frac{(\sqrt{\gamma_t(\sigma_e)D_p}+A)^2}{\epsilon^2} \Bigg\rceil
\end{align}
and since $\gamma_t(\sigma_e)$ strictly increases with $e$ we can say that
\begin{align}
	 \Bigg\lceil\frac{2}{1-q} \times \frac{(\sqrt{\gamma_t(\sigma_{e_{max}})D_p}+A)^2}{ \epsilon^2} \Bigg\rceil \times e_{max}
\end{align}
episodes suffices.
\end{proof}

\newpage
\subsection{Learning to Navigate}
In this section we show that \Alg{} ``learns to navigate'', minimizing the least-square error in \lsvi{} across timesteps.

\begin{proposition}[Learning to Navigate]
\label{prop:LearningNavigate}
Assume that\footnote{Both assumptions are satisfied by the assumptions of the main theorem.}:
\begin{enumerate}
	\item $\IBEE(\mathcal Q_t,\mathcal Q_{t+1}) \leq \frac{\epsilon}{2H}$ \quad (this is always satisfied by our assumptions on $\epsilon$)   \\
	\item $\(\frac{\nu}{\epsilon}\)^2 \geq 2 \times \MagicValue{p}$ \quad (this is also always satisfied by our assumptions on $\epsilon$)
\end{enumerate}
Then after 
\begin{align}
	\Bound
\end{align}
episodes, outside of the failure event it holds that 
\begin{align}
		\sum_{t=1}^{H} \Big[ \IBEE(\mathcal Q_{t},\mathcal Q_{t+1}) + \sqrt{\alpha_{t}} \| \overline \phi_{\pi,t} \|_{\Sigma_{t}^{-1}} \Big] \leq \epsilon, \quad \forall \pi
\end{align}
and in particular
\begin{align}
		 \IBEE(\mathcal Q_{t},\mathcal Q_{t+1}) + \sqrt{\alpha_{t}} \| \overline \phi_{\pi,t} \|_{\Sigma_{t}^{-1}}  \leq \frac{\epsilon}{H}, \quad \forall \pi, t\in[H].
\end{align}

\end{proposition}
\begin{proof}
	We proceed by induction over timesteps / phases $p$:	\begin{inductivehypothesis}[Main Inductive Hypothesis]
	\label{hypo:master}
	In phase $p \in [H]$ it holds that
	\begin{enumerate}
	\item $\sum_{t=1}^{p-1} \Big[ \IBEE(\mathcal Q_{t},\mathcal Q_{t+1}) + \sqrt{\alpha_{t}} \| \overline \phi_{\pi,t} \|_{\Sigma_{t}^{-1}} \Big] \leq \frac{p-1}{H}\epsilon$ \quad (this ensures accuracy in \lsvi{}) \\
	\item $	\lambda_{min}(\Sigma_{t})  \geq 4H^2\alpha_t \quad t \in [p-1]$ \quad (this ensures boundness of the iterates in \lsvi{}) 
	\end{enumerate}
\end{inductivehypothesis}
The inductive hypothesis vacuously holds for $p = 1$ (there is nothing to check).
Now we show the inductive step. Assume the inductive hypohesis holds for a generic $p-1$, we want to show it still holds for $p$.
A direct application of \fullref{lem:LearningALevel} gives ($\Sigma_{p}$ is the covariance matrix after learning has completed):
\begin{align}
\sqrt{\alpha_p}\| \overline \phi_{\pi,p} \|_{\Sigma^{-1}_{p}}	\overset{\text{\cref{lem:LinearBanditBonus}}}{=} \Mdef{\alpha_p}{\Sigma_{p}}{p} \leq \frac{\epsilon}{2H}
\end{align}
Adding
\begin{align}
	\IBEE(\mathcal Q_t,\mathcal Q_{t+1}) \leq \frac{\epsilon}{2H}
\end{align}
to both sides and adding the result to the equation in the inductive hypothesis proves the inductive step. The final number of episodes follows from summing the episodes needed in every phases according to \fullref{lem:LearningALevel}.
\end{proof}

\newpage
\subsection{Solution Reconstruction (\emph{Main Result})}
\label{sec:SolutionReconstruction}
In this section we present our main result in a more formal way than in the main text; throughout the appendix the symbols are generally reported in \cref{tab:MainNotation}. 

First, let us define the reward classes.
\begin{definition}[Reward Classes]
\label{def:RewardClasses}
Consider an MDP $\mathcal M(\StateSpace,\ActionSpace,p,\cdot,H )$ without any reward function. Fix a misspecification function $\Delta^{r}_{t}(\cdot,\cdot,): \StateSpace\times\ActionSpace \rightarrow \R $  for every $t \in [H]$ which can depend on the state and action pair, and is subject to the constraint
\begin{align}
\label{eqn:174}
	\forall (\pi,t) \quad |\E_{x_t \sim \pi} \Delta^{r}_t(x_t,\pi_t(x_t))| \defeq |\overline \Delta^{r}_{\pi,t}| \leq E_t.
\end{align}

Define the following class $\mathfrak{I}$ (\emph{Implicit Regularity}) of (expected) reward functions  $(r_1,\dots,r_H)$ on $\mathcal M$, parameterized by $(\theta^r_1,\dots,\theta^r_H)$ and satisfying $ \forall (s,a,t,\pi) \in \StateSpace\times\ActionSpace\times[H]\times \Pi$ (here $\Pi$ is the policy space):
\begin{enumerate}
	\item $ r_t(s,a) = \phi_t(s,a)^\top \theta_t^r + \Delta^{r}_t(s,a) $
	\item $|\Delta^{r}_t(s,a) | \leq 1$ 
	\item $|\E_{x_t \sim \pi} r_t(x_t,\pi_t(x_t))| \leq \frac{1}{H}$
\end{enumerate}	

In addition, define the following class $\mathfrak{E}$ (\emph{Explicit Regularity}) of (expected) reward functions  $(r_1,\dots,r_H)$ on $\mathcal M$ parameterized by $(\theta^r_1,\dots,\theta^r_H)$ satisfying $ \forall (s,a,t,\pi) \in \StateSpace\times\ActionSpace\times[H]\times \Pi$:
\begin{enumerate}
	\item $ r_t(s,a) = \phi_t(s,a)^\top \theta_t^r + \Delta^{r}_t(s,a) $
	\item $|\Delta^{r}_t(s,a) | \leq 1$ 
	\item $\|\theta^r_t\|_2 \leq \frac{1}{H}$.
\end{enumerate}	
\end{definition}

Under explicit regularity the bound on $\| \theta^r_t \|_2$ constrains the maximum value the reward can take; instead, under implicit regularity we do not have such requirement, as only the expectation is controlled. This implies the local reward can be much larger than the expectation, making this a much harder setting.

We are now ready to present the main result formally.
\begin{theorem}[Restating \cref{main.thm:MainResult} formally]
\label{thm:MainResult}
Consider an MDP $\mathcal M$ and a feature extractor $\phi$ satisfying $ \| \phi_t(s,a) \|_2 \leq 1$ for every $(s,a) \in \StateSpace\times\ActionSpace$ and fix two classes of reward functions $\mathfrak{I}$ and $\mathfrak{E}$ according to \fullref{def:RewardClasses}. Set $\epsilon$ to satisfy $\epsilon \geq \widetilde \Omega(d_tH(\IBEE(\mathcal Q_{t},\mathcal Q_{t+1}) + E_t))$ and $\epsilon \leq \widetilde O(\nu_{min}/\sqrt{d_t})$ for all $t \in [H]$.  

\Alg{} always terminates after $\Bound$ episodes (with probability one), returning a dataset $\mathcal D = \{(s_{tk},a_{tk},s^+_{t+1,k}) \}^{k=1,\dots,n(t)}_{t=1,\dots,H}$ of the collected state-action-successor states $(s_{tk},a_{tk},s^+_{t+1,k})$ in episode $k \in [n(t)]$ for each timestep $t \in [H]$.

Now consider any reward function $r \in \mathfrak{E}$ or $r \in \mathfrak{I}$ and the MDP induced by that reward function $\mathcal M(\StateSpace,\ActionSpace,p,r,H )$, and replace each tuple  $(s_{tk},a_{tk},s^+_{t+1,k}) \in \mathcal D$ with $(s_{tk},a_{tk},r_{tk},s^+_{t+1,k})$ where $r_{tk}$ satisfies

\begin{align}
\label{eqn:Rnoise}
	r_{tk} = r_t(s_{tk},a_{tk}) + \eta^r
\end{align}
where $\eta^r$ is 1-sub-Gaussian noise. 

Then with probability at least $1-\delta$, the batch \lsvi{} algorithm run on $\mathcal D$ (see \cref{alg:lsviBatch}) returns a policy $\pi$ such that on $\mathcal M$
\begin{align}
		\E_{x_1 \sim \rho} (V_1^\star - V^\pi_1)(x_1) \leq \frac{\epsilon}{\nu_{min}}.
\end{align}
if $r \in \mathfrak{I}$ and
\begin{align}
		\E_{x_1 \sim \rho} (V_1^\star - V^\pi_1)(x_1) \leq \epsilon.
\end{align}
if $r \in \mathfrak{E}$. 
\end{theorem}

We have expressed the theorem in its full generality, but if the reward function is prescribed a posteriori through an oracle then we expect the noise $\eta^r$ in \cref{eqn:Rnoise} to be absent. In general, if the reward function is prescribed a posteriori then it should be prescribed as a linear function (in the chosen features) to avoid any additional error in the \lsvi{} procedure. Finally the reward misspecification $\Delta^r_t(\cdot,\cdot)$ can depend on the parameter $\theta$ if it is a Lipshitz function of $\theta$. Alternatively, if it is a discontinuous function of $\theta$ then same-order guarantees are still recovered if \cref{eqn:174} is replaced with 	$\forall (s,a,t) \quad | \Delta^{r}_t(s,a)| \leq E_t$.

\begin{proof}(of the main result) 
Let $n(t)$ the number of samples collected at each level (notice that we only store one sample every trajectory, so the number of samples equals the number of trajetories / number of episodes), according to \fullref{lem:LearningALevel}. Using the assumptions on $\epsilon$ (these conditions are used in the good event for \lsvi{} in \fullref{def:GoodEventLSVI}) we can ensure:
\begin{align}
	 \sqrt{n(t)}E_t =  \sqrt{\frac{n(t)}{\alpha_t}}E_t \sqrt{\alpha_t} = \widetilde O\( \frac{d_tH\sqrt{\alpha_t}}{\sqrt{\alpha_t} \epsilon} \)E_t \sqrt{\alpha_t} \leq \sqrt{\alpha_t}/3 \\
	  \sqrt{n(t)}\IBEE(\mathcal Q_{t},\mathcal Q_{t+1}) =  \sqrt{\frac{n(t)}{\alpha_t}}\IBEE(\mathcal Q_{t},\mathcal Q_{t+1}) \sqrt{\alpha_t} = \widetilde O\( \frac{d_tH\sqrt{\alpha_t}}{\sqrt{\alpha_t} \epsilon} \)\IBEE(\mathcal Q_{t},\mathcal Q_{t+1}) \sqrt{\alpha_t} \leq \sqrt{\alpha_t}/3.
\end{align}
We assume we are in the good event\footnote{We sometime say we are outside of the failure event to mean  we are in the good event for \Alg{}, see \fullref{def:GoodEventFrancis}. In particular, the computation in \fullref{lem:ProbGoodEventFrancis} together with the proof in \fullref{lem:LearningALevel} would provide values for $\delta''$ and for the constants $c_e,c_{\alpha},c_\sigma$ if carried out explicitly.} for \Alg{}, see \fullref{def:GoodEventFrancis}, which occurs with probability $1-\delta$ according to \fullref{lem:ProbGoodEventFrancis}. We apply \fullref{prop:LearningNavigate}, which gives the stated number of episodes to termination and the condition satisfied by the samples in the dataset $\mathcal D$ (through the covariance matrices $\Sigma^{-1}_t$):
\begin{align*}
		\sum_{t=1}^{H} \Big[ \IBEE(\mathcal Q_{t},\mathcal Q_{t+1}) + \sqrt{\alpha_{t}} \| \overline \phi_{\pi,t} \|_{\Sigma_{t}^{-1}} \Big] & \leq \epsilon, \quad \forall \pi \\
		\IBEE(\mathcal Q_{t},\mathcal Q_{t+1}) + \sqrt{\alpha_{t}} \| \overline \phi_{\pi,t} \|_{\Sigma_{t}^{-1}} & \leq \frac{\epsilon}{H}, \quad \forall \pi, t\in[H].
\numberthis{\label{eqn:180}}
\end{align*}
Now, under \emph{implicit regularity} \fullref{lem:RewardBoundness} ensures (the lemma requires $E_t \leq \frac{1}{H}$, which is always satisfied since we must have $\epsilon < 1$ to produce any useful result, and from the theorem hypothesis $E_t  \leq  \epsilon/(d_t H)  \leq 1/ H$)
\begin{align}
	\| \theta^R_t \|_2 \leq \frac{2}{H\nu_{min}} \defeq \frac{R}{H}.
\end{align}
Finally, \fullref{prop:BatchLSVI} ensures that \lsvi{} in \cref{alg:lsviBatch} returns a value function $\Vhat$ and policy $\pihatstar$ such that
\begin{align*}
	\E_{x_1 \sim \rho}  \( \Vstar_1 - \Vhat_1\)(x_1)  & \leq \sum_{t=1}^H \Bigg[ 2E_t + R\( \IBEE(\mathcal Q_t,\mathcal Q_{t+1}) + \sqrt{\alpha_t}\| \overline \phi_{\pistar,t} \|_{\Sigma^{-1}_t}\) \Bigg] \\
 \E_{x_1 \sim \rho }\( \Vhat_{1} - V_1^{\pihatstar}\)(x_1)  & \leq \sum_{t=1}^H \Bigg[ 2E_t + R\(\IBEE(\mathcal Q_t,\mathcal Q_{t+1}) + \sqrt{\alpha_t}\| \overline \phi_{\pihatstar,t} \|_{\Sigma^{-1}_t}\)\Bigg].
\numberthis{\label{eqn:FinalSplitInMain}}
\end{align*}
Using \cref{eqn:180} (and recalling $E_t \leq \epsilon$ by hypothesis of the theorem) to further simplify it we obtain:
\begin{align*}
	\E_{x_1 \sim \rho}  \( \Vstar_1 - \Vhat_1\)(x_1)  & \leq  2R\epsilon \\
 \E_{x_1 \sim \rho }\( \Vhat_{1} - V_1^{\pihatstar}\)(x_1)  & \leq 2R\epsilon .
\end{align*}
Summing the two expression gives:
\begin{align*}
	\E_{x_1 \sim \rho}  \( \Vstar_1  - V_1^{\pihatstar}\)(x_1)  & \leq  4R\epsilon.
\end{align*}
Rescaling $\epsilon$ by $4$ and substituting the value for $R$ gives the thesis under \emph{implicit regularity}. 

Under \emph{explicit regularity} the steps are the same, but now
\begin{align}
	\| \theta^r_t\|_2 \leq \frac{1}{H} \defeq \frac{R}{H}
\end{align}
is explicitly prescribed, and the thesis immediately follows.
\end{proof}

The generality of the main result allows us to immediately obtain the following corollary:

\begin{corollary}[Learning a Prescribed Reward Function during the Execution]
Under the same assumptions as \cref{thm:MainResult}, assume the reward function $r \in \mathfrak{E}$ or $r \in \mathfrak{I}$ is prescribed before the execution of \Alg{} and
\begin{align}
	r_{tk} = r_t(s_{tk},a_{tk}) + \eta^r
\end{align}
where $\eta^r$ is 1-sub-Gaussian noise. Assume $(s_{tk},a_{tk},r_{tk},s^+_{t+1,k})$ is stored in the dataset $\mathcal D$.

Then with probability at least $1-\delta$, the batch \lsvi{} algorithm run on $D$ (see \cref{alg:lsviBatch}) returns a policy $\pi$ such that on $\mathcal M$
\begin{align}
		\E_{x_1 \sim \rho} (V_1^\star - V^\pi_1)(x_1) \leq \frac{\epsilon}{\nu_{min}}.  
\end{align}
if $r \in \mathfrak{I}$ and
\begin{align}
		\E_{x_1 \sim \rho} (V_1^\star - V^\pi_1)(x_1) \leq \epsilon. 
\end{align}
if $r \in \mathfrak{E}$.
\end{corollary}

\newpage
\subsection{Computational Complexity}
\label{sec:ComputationalComplexity}

Theorem \ref{main.thm:MainResult} gives a bound on the number of episodes to termination. In every episode, a multivariate normal vector is sampled (which can be done efficiently) and \lsvi{} is invoked. 

Assume $d_1=\dots =d_H=d$ for simplicity; a naive implementation would factorize and store the new covariance matrix at the end of a phase (total of $\widetilde O(Hd^3)$ work across all phases); after this, computing the $\widehat \theta_t$'s requires $\widetilde O\( H(d^2 + A d) \times n_{episodes} \)$ computations at every episode where $n_{episodes}$ is the total number of episodes at termination given in \cref{main.thm:MainResult}.

\newpage

\begin{definition}[Good Event for \Alg{}]
\label{def:GoodEventFrancis}
We say the good event for \Alg{} occurs if for all timesteps $t \in [H]$ or phases $p \in [H]$ and episodes $k$ in that phase the following bounds\footnote{Some symbols, like $i_{max},k_{max}$  are defined  directly in the lemma where the bound is used.} jointly hold and we are in the good event for \lsvi{} (see \fullref{def:GoodEventLSVI}).
\begin{align}
\Big| \frac{1}{i_{max}}\sum_{i=1}^{i_{max}} \zeta_{pk(e,i)} \Big| & \leq \sqrt{\frac{2(2\Lphi\Radius_t)^2\ln\( \frac{1}{\delta''} \)}{i_{max}}} = \frac{\Adef}{\sqrt{i_{max}}} \defeq \frac{A}{\sqrt{i_{max}}} \\
 	  \| \xi_{t,k(e,i)} \|_{\Sigma_{t,k(e,i)}} & \leq \sqrt{\gamma_{t}(\sigma)} \defeq  \twonormbound	\\
 	   	  \| \xi_{t,k(e,i)} \|_{2} & \leq \gammadef	\\
 	   \frac{1}{k_{max}} \sum_{k=1}^{k_{max}} \1\{\mathcal{E}_k \mid \mathcal C_k \} & \geq \( 1-q \) - \sqrt{\frac{2\ln(\frac{1}{\delta''})}{k_{max}}} 
\end{align}
\end{definition}

\begin{lemma}[Probability of Good Event for \Alg{}] 
\label{lem:ProbGoodEventFrancis}
There exists a parameter $\delta'' = \frac{\delta}{poly(d_1,\dots,d_H,H,\frac{1}{\epsilon})}$, such that the good event of \cref{def:GoodEventFrancis} holds with probability at least $1-\delta$.
\end{lemma}
\begin{proof}
	The first and fourth inequality follow from \fullref{prop:Azuma}. The second and third inequality follow from \fullref{lem:MaxDeviationLast}. In particular, a union bound over the statements, over $H$ and over the number of episodes ensures all statements jointly hold at any point during the execution of the program; from this, the value for $\delta''$ can be determined.
\end{proof}

\newpage
\section{Lower Bound}
\label{sec:LowerBound}
We sketch the lower bound to highlight that explorability is required.

\begin{proposition}[Lower Bound on Explorability Dependence under Implicit Regularity]
There exists an MDP and a feature map $\phi_t: (s,a) \mapsto \phi_t(s,a) \in \R^2$ with explorability parameter $\nu_{min}$ and a reward function such that:
\begin{align}
\forall (\pi,t) \quad 	r_t(s,a) = \phi_t(s,a)^\top \theta^r_t, \quad | \E_{x_t \sim \pi} r_t(x_t,\pi_t(x_t)) | \leq 1
\end{align}
and yet no reinforcement learning agent without knowledge of $\theta^r$ can return an $\epsilon$-optimal policy for $\epsilon \leq \nu_{min} \leq \frac{1}{4}$ in less than $\Omega(1/(\epsilon \nu_{min})^2)$ trajectories with probability higher than $2/3$.
\end{proposition}
Notice that the proposition above is for a fixed (but unknown) deterministic reward function; this is thus a special case of the reward-free learning setting we consider, implying that the hardness is due to the implicit regularity conditions rather than to reward-free learning.

The proof essentially uses a multi-armed bandit lower bound where the noise is $1/\nu_{min}$-sub-Gaussian and is created using the MDP dynamics (since the reward is deterministic).
\begin{proof}
We construct the MDP as follows: there is a single starting state $s_{start}$ with two actions $a_{L}$ and $a_{R}$ and the identity feature $\phi_{1}(s_{start},a_{L}) = e_1, \phi_{1}(s_{start},a_{R}) = e_2$, where $e_1,e_2$ are canonical vectors in $\R^2$. Now fix a scalar $\epsilon \in [-\frac{\nu_{min}}{2} , \frac{\nu_{min}}{2} ]$:
\begin{enumerate}
\item action $a_{L}$ gives an immediate reward $-1/2$ and leads to state $s_{L1}$ with probability $\frac{1}{2}+ \nu_{min}$ and to $s_{L2}$ with probability $\frac{1}{2}-\nu_{min}$. The feature map reads $\phi_2(s_{L1}) = e_1$ and $\phi_2(s_{L2}) = -e_1$ in the only action available in each state.
\item action $a_{R}$ gives an immediate reward $-1/2$ and leads to state $s_{R1}$ with probability $\frac{1}{2} + \nu_{min} + \epsilon$ and to $s_{R2}$ with probability $\frac{1}{2} - \nu_{min} - \epsilon$. The feature map reads $\phi_2(s_{R1}) = e_2$ and $\phi_2(s_{R1}) = -e_2$
\end{enumerate}
In this MDP there are only two distinct policies: $\pi_L$ that selects $a_L$ first and then the only available action in either $s_{L1}$ or $s_{L2}$, and $\pi_R$ that selects $a_R$ first and then the only available action in either $s_{R1}$ or $s_{R2}$. Therefore, this is equivalent to a \emph{multiarmed} bandit problem with reward  $-1/2 + \overline \phi_{\pi_L,2}^\top \theta^r_2$ for $\pi_L$ and $-1/2 + \overline \phi_{\pi_R,2}^\top \theta^r_2$ for $\pi_2$. The minimum explorability coefficient is ($\nu_1 = 1$ at timestep $1$)
\begin{align}
 \min_{\theta \neq 0}\max_{\pi} \overline\phi_{\pi,2}^\top\frac{\theta}{\|\theta\|_2} = 	\Bigg[ \( \frac{1}{2} +  \nu_{min} - \frac{\nu_{min}}{2} \) -  \( \frac{1}{2} -  \nu_{min} +\frac{\nu_{min}}{2}  \)\Bigg]e_2^\top e_2 =  \nu_{min}
\end{align}
corresponding to policy $\pi_R$ (this can be computed by inspection; notice that $\pi_L$ yields the same $\nu_{min}$).
Now consider the reward parameter $\theta^r_2 = 1/\nu_{min}\times [1/2,1/2]$; 
the expected reward at timestep $2$ under policy $\pi_R$ is $\E_{x_2 \sim \pi_L} r_2(x_2) = \nu_{min}\times \frac{1}{2\nu_{min}} \leq 1$ which satisfies the assumptions of the lemma. 
At the same time $\E_{x_2 \sim \pi_R} r_2(x_2) = (\nu_{min} + 2\epsilon) \times \frac{1}{2\nu_{min}} \leq 1$. 
This implies the random return $ -1/2 + \phi_2(s)^\top \theta_2$ with $s \sim p_1(s_{start},a_{L})$ is a scaled and shifted Bernoulli random variable with mean zero, taking the values $-1/2 + \frac{1}{2\nu_{min}}$ and $-1/2 -\frac{1}{2\nu_{min}}$. Since the standard deviation of this random variables (with $\nu_{min} \leq \frac{1}{4}$) is $\Omega(1/\nu_{min})$, this random variable must be  $\Omega(1/\nu_{min})$-sub-Gaussian\footnote{See for example exercise 2.5 in \cite{wainwright2019high}.}. The same reasoning applies to $ -1/2 + \phi_2(s)^\top \theta_2$ with $s \sim p_1(s_{start},a_{R})$. Notice that both expectations are at most $1$. 

Solving this class of problems (parameterized by $\epsilon$), i.e., identifying an $|\epsilon|/2$-optimal policy  is equivalent to solving a multiarmed bandit problem with 2 actions (corresponding to the policies $\pi_1$ and $\pi_2$). This construction is exactly the same as theorem 2 from \cite{krishnamurthy2016pac} with  shifted Bernoulli random variables that are scaled by the inverse explorability coefficient $1/\nu_{min}$. This implies that a sample complexity $\Omega(1/(\nu_{min}|\epsilon|)^2)$ is required to output an $|\epsilon|/2$-optimal policy with probability $> 2/3$.
\end{proof}

\newpage
\section{Support Lemmas}

\begin{lemma}[Reward Boundness]
\label{lem:RewardBoundness}	
If we assume that
\begin{align}
\forall \pi \quad \quad & |	\E_{x_t \sim \pi} r_t(x_t,\pi_t(x_t))| \leq \frac{1}{H} \\ \text{and} \quad \exists \theta_t^r \in \R^{d_t} & \quad \text{such that} \quad \quad |\E_{x_t \sim \pi} r_t(x_t,\pi_t(x_t)) - \overline \phi_{\pi,t}^\top\theta^r_t| \leq E_t \leq \frac{1}{H}
\end{align}
then it follows that
\begin{align}
\| \theta^r_t \|_2 \leq \frac{2}{H\nu_t}.
\end{align}
\end{lemma}
\begin{proof}
From the hypothesis it follows
\begin{align}
	\frac{2}{H} \geq |\overline \phi_{\pi,t}^\top \theta^r_t| = \| \theta^r_t \|_2 \times |\overline \phi_{\pi,t}^\top \frac{\theta^r_t}{\| \theta^r_t \|_2 }|
\end{align}
in particular this must hold for the policy $\pi$ that maximizes the above display. Therefore, after taking $\max_{\pi}$, take $\min_{\|\theta\|_2 = 1}$ to obtain (using \fullref{main.def:Explorability}):
\begin{align}
	\geq \| \theta^r_t \|_2 \times \min_{\|\theta\|_2 = 1} \max_{\pi} |\overline \phi_{\pi,t}^\top \theta| = \| \theta^r_t \|_2 \nu_t.
\end{align}
Rearranging
\begin{align}
\| \theta^r_t \|_2 \leq \frac{2}{H\nu_t}.
\end{align}
\end{proof}

\newpage 
\subsection{High Probability Bounds}
\begin{lemma}[Transition Noise High Probability Bound]
\label{lem:TransitionsHPbound}
	If $\lambda = 1$ and $R = 2\Lphi\Radius_{t+1}$ with probability at least $1-\delta'$ it holds that $\forall V_{t+1}\in\mathcal V_{t+1}$:

	\begin{align}
	\Big\|\sum_{i=1}^{k-1} \phi_{ti} \( V_{t+1}(s^{+}_{t+1,k}) - \E_{s' \sim p(s_{tk},a_{tk})} V_{t+1}(s') \) \Big\|_{\Sigma^{-1}_{t}} \leq \sqrt{\beta^t_{t}}
	\end{align}
	where:
	\begin{align}
	\sqrt{\beta^t_{t}} \defeq \sqrtbetadef.	
	\end{align}
\end{lemma}
\begin{proof}
Since the statement needs to hold for every $ V_{t+1}\in\mathcal V_{t+1}$, we start by constructing an $\epsilon$-cover for set $\mathcal V_{t+1}$ using the supremum distance. To achieve this, we construct an $\epsilon$-cover for the parameter $\theta \in \mathcal B_{t+1}$ using the ``Covering Number of Euclidean Ball'' lemma in \citep{zanette2020learning}. This ensures that there exists a set $\mathcal D_{t+1} \subseteq \mathcal B_{t+1}$, containing $(1+2\Radius_{t+1}/\epsilon')^{d_{t+1}}$ vectors $\overset{\triangle}{\theta}_{t+1}$ that well approximates any $\theta_{t+1}\in\mathcal B_{t+1}$:
\begin{align}
\exists \mathcal D_{t+1}\subseteq \mathcal B_{t+1}  \quad \textrm{such that} \quad \forall \theta_{t+1} \in \mathcal B_{t+1}, \quad \exists \overset{\triangle}{\theta}_{t+1} \in \mathcal D_{t+1} \quad \textrm{such that} \quad \| \theta_{t+1} - \overset{\triangle}{\theta}_{t+1} \|_2 \leq \epsilon'.
\end{align}
Let $\overset{\triangle}{V}_{t+1}(s) \defeq \max_{a} \phi_{t+1}(s,a)^\top \overset{\triangle}{\theta}$, where $\overset{\triangle}{\theta} = \argmin_{\theta' \in \mathcal D_{t+1}} \|\theta' - \theta \|_2 $ and consider $V_{t+1} \in \mathcal V_{t+1}$.
For any fixed $s \in \StateSpace$ we have that:
\begin{align*}
|\big(V_{t+1} - \overset{\triangle}{V}_{t+1} \big)(s) |& = |\max_{a'} \phi_{t+1}(s,a')^\top \theta_{t+1} - \max_{a''} \phi_{t+1}(s,a'')^\top \overset{\triangle}{\theta}_{t+1} | \\
& \leq \max_{a} 
| \phi_{t+1}(s,a)^\top\big( \theta_{t+1} - \overset{\triangle}{\theta}_{t+1} \big) | \\
& \leq \max_a \| \phi_{t+1}(s,a)\|_2\| \theta_{t+1} - \overset{\triangle}{\theta}_{t+1} \|_2 \\
& \leq \Lphi \epsilon'.
\numberthis{\label{eqn:coveringargument}}
\end{align*}
By using the triangle inequality we can write:
\begin{align*}
	& \Big\|\sum_{i=1}^{k-1}  \phi_{ti} \( V_{t+1}(s^{+}_{t+1,k}) - \E_{s' \sim p(s_{tk},a_{tk})} V_{t+1}(s') \) \Big\|_{\Sigma^{-1}_{t}} & \\
	& \leq \Big\|\sum_{i=1}^{k-1}  \phi_{ti} \( \overset{\triangle}{V}_{t+1}(s^{+}_{t+1,k}) - \E_{s' \sim p(s_{tk},a_{tk})} \overset{\triangle}{V}_{t+1}(s') \) \Big\|_{\Sigma^{-1}_{t}}+  \\
	& + \Big\| \sum_{i=1}^{k-1}  \phi_{ti} \(  \E_{s' \sim p(s_{tk},a_{tk})} \overset{\triangle}{V}(s') - \E_{s' \sim p(s_{tk},a_{tk})} V_{t+1}(s') \) \Big\|_{\Sigma^{-1}_{t}} \\
	& + \Big\| \sum_{i=1}^{k-1}  \phi_{ti} \(  V_{t+1}(s^{+}_{t+1,k})  - \overset{\triangle}{V}_{t+1} (s^{+}_{t+1,k}) \) \Big\|_{\Sigma^{-1}_{t}}.
\numberthis{\label{eqn:Equation1}}
\end{align*}
Each of the last two terms above can be written for some $b_i$'s (different for each of the two terms) as $ \Big\| \sum_{i=1}^{k-1} \phi_{ti} b_i \Big\|_{\Sigma^{-1}_{tk}}$. The projection lemma, (lemma 8 from \cite{zanette2020learning}) ensures:
\begin{align}
	 \Big\| \sum_{i=1}^{k-1} \phi_{ti} b_i \Big\|_{\Sigma^{-1}_{t}} \leq \Lphi \epsilon' \sqrt{k}
\end{align}
We have used \cref{eqn:coveringargument} to bound the $b_i$'s.
Now we examine the first term of the rhs in equation in \cref{eqn:Equation1}. In particular, we bound that term for a generic $\overset{\triangle}{V}_{t+1}$ and then do a union bound over all possible $\overset{\triangle}{V}_{t+1}$, which are generated by finitely many $\overset{\triangle}{\theta}_{t+1} \in \mathcal D_{t+1}$ as explained before. We obtain that:
\begin{align}
\Pro\Bigg( \bigcup_{\overset{\triangle}{\overline \theta}_{t+1} \in \mathcal D_{t+1}} C(\overset{\triangle}{\overline \theta}_{t+1}) \Bigg) \leq \sum_{\overset{\triangle}{\overline \theta}_{t+1} \in \mathcal D_{t+1}}\Pro\Bigg(  C(\overset{\triangle}{\overline \theta}_{t+1}) \Bigg) \leq  (1+2\Radius_{t+1}/\epsilon')^{d_{t+1}}\delta'' \defeq \delta'
\end{align}
where $C$ is the event reported below (along with $\delta''$) and the last inequality above follows from Theorem 1 in \citep{Abbasi11} (the random variables $\overset{\triangle}{V}_{t+1}(\cdot)$ and $\Vhat_{t+1}(\cdot)$ are $R = 2\Lphi\Radius_{t+1}$-subgaussian by construction):
\begin{align}
C(\overset{\triangle}{\overline \theta_{t+1}})  \defeq \Bigg\{ \Big\|\sum_{i=1}^{k-1} \phi_{ti} \( \overset{\triangle}{V}_{t+1,i} - \E_{s' \sim p(s_{tk},a_{tk})} \overset{\triangle}{V}(s') \) \Big\|^2_{\Sigma^{-1}_{t}} > 2\times(R)^2\ln\(\frac{\det(\Sigma_{t})^{\frac{1}{2}}\det\( \lambda I\)^{-\frac{1}{2}}}{\delta''} \) \Bigg\}.
\end{align}
In particular, we set
\begin{align}
\delta'' = \frac{\delta'}{(1+2\Radius_{t+1}/\epsilon')^{d_{t+1}}}	
\end{align}
from the prior display and so with probability $1-\delta'$ (after a union bound over all possible $\overset{\triangle}{\theta_{t+1}} \in \mathcal D_{t+1}$) we have upper bounded \cref{eqn:Equation1} by:
\begin{align}
 R\sqrt{2\ln\(\frac{\det(\Sigma_{t})^{\frac{1}{2}}\det\( \lambda I\)^{-\frac{1}{2}}(1+2\Radius_{t+1}/\epsilon')^{d_{t+1}}}{\delta'} \)} + 2\Lphi\epsilon'\sqrt{k}.	
\end{align}
If we now pick
\begin{align}
\epsilon' = \frac{R}{2\Lphi\sqrt{k}}
\end{align}
we get:
\begin{align}
R\sqrt{2\ln\(\frac{\det(\Sigma_{t})^{\frac{1}{2}}\lambda^{-\frac{d_t}{2}}(1+2\Radius_{t+1}/\epsilon')^{d_{t+1}}}{\delta'} \)}+R  \\
= \sqrt{2}R\sqrt{\frac{1}{2}\ln\(\det(\Sigma_{t})\) -\frac{d_t}{2}\ln\(\lambda\) + d_{t+1}\ln(1+2\Radius_{t+1}/\epsilon') + \ln\(\frac{1}{\delta'} \)} + R
\end{align}
Finally, using the Determinant-Trace Inequality (see lemma 10 of \citep{Abbasi11}) we obtain $\det(\Sigma_{tk}) \leq \( \lambda + \Lphi^2 k/d_t\)^{d_t}$ and so (with $\lambda = 1$) 
\begin{align}
	\leq \sqrtbetadef 
	\defeq \sqrt{\beta^t_{t}}.
\end{align}
\end{proof}

\newpage
\subsection{Known Results}

\begin{lemma}[Large Deviation Multivariate Normal]
\label{lem:MaxDeviationLast}
Let $\Sigma \in \R^{d\times d}$ be an spd matrix with minimum eigenvalue $\lambda > 0$ and let
\begin{align}
	 \xi \sim \mathcal N\(0,\sigma\Sigma^{-1} \)
\end{align}
for a positive scalar $\sigma$. For any fixed $\phi \in \R^d$ with probability at least $1-\delta'$:
\begin{align}
|\phi^\top \xi|^2 \leq  \frac{\sigma \|\phi\|_2^2}{\lambda} \( 2d\ln\frac{2d}{\delta'} \)
\end{align}
and so by choosing $\phi = \frac{\xi}{\|\xi\|_2}$ when $\xi \neq 0$ it holds that
\begin{align}
\|\xi\|_2 \leq  \sqrt{\frac{\sigma}{\lambda} \( 2d\ln\frac{2d}{\delta'} \)}.
\end{align}
Under the same event it holds that
\begin{align}
	\| \xi \|_{\Sigma} \leq \sqrt{\sigma\( 2d\ln\frac{2d}{\delta'} \)}.
\end{align}
\end{lemma}
\begin{proof}
If
\begin{align}
	 \xi \sim \mathcal N\(0, \sigma \Sigma^{-1} \)
\end{align}
it follows that
\begin{align}
	\frac{1}{\sqrt{\sigma}}\Sigma^{\frac{1}{2}} \xi \sim \mathcal N\(0,I\)
\end{align}
where $I$ is the identity matrix on $\R^d$. Therefore
\begin{align}
\frac{1}{\sigma}\| \xi \|^2_\Sigma = \( \frac{1}{\sqrt{\sigma}} \xi^\top\Sigma^{\frac{1}{2}}\)^\top \( \frac{1}{\sqrt{\sigma}}\Sigma^{\frac{1}{2}}\xi\) \sim \chi^2_{d}
\end{align}
where $\chi^2_{d}$ is the chi-square distribution with $d$ degrees of freedom.
From \fullref{lem:Xsquare} we can compute a high probability bound for the above random variable (this also proves the last statement):
\begin{align}
|\phi^\top \xi|^2 \leq \|\phi\|^2_{\Sigma^{-1}}  \|\xi\|^2_{\Sigma} \leq \| \phi\|_2^2\frac{\sigma}{\lambda}  \frac{1}{\sigma}\|\xi\|^2_{\Sigma}  \leq \frac{\sigma \| \phi\|_2^2}{\lambda} \( 2d\ln\frac{2d}{\delta'} \)
\end{align}
with probability at least $1-\delta'$.  
\end{proof}

\begin{lemma}[$\chi$-square lemma]
\label{lem:Xsquare}
	Let $X^2\sim \chi^2_d$ be a random variable that follows the chi-square distribution with $d$ degrees of freedom. With probability at least $1-\delta'$
\begin{align}
X^2 \leq 2d\ln\frac{2d}{\delta'}.
\end{align}
\end{lemma}
\begin{proof}
	Let $X_i\sim \mathcal N(0,1), i \in [d]$. If $X_i \in [-a,+a], \forall i \in [d]$ then it must follow that $\sum_{i \in [d]} X_i^2 \leq d a^2$. Thus:$$
\Pro(X^2 = \sum_{i \in [d]} X^2_i \geq d a^2) \leq \Pro(\exists i \in [d], X_i \not \in [-a,a]) = \Pro(\cup_{i \in [d]} X_i \not \in [-a,a]) \leq  d\Pro(X_i \not \in [-a,a]) \leq 2de^{-a^2/2}. 
$$
Requiring the rhs above to be $\leq \delta'$ gives
$$
a^2 = 2\ln\frac{2d}{\delta'}.
$$
\end{proof}

\begin{lemma}[Azuma-Hoeffding Inequality]
\label{prop:Azuma}
Let $X_i$ be a martingale difference sequence such that $X_i \in [-A,A]$ for some $A>0$. Then with probability at least $1-\delta'$ it holds that:
\begin{align}
 \Big| \sum_{i=1}^{n} X_i \Big| \leq \sqrt{2A^2n\ln\( \frac{1}{\delta'} \)}.
\end{align}
\end{lemma}
\begin{proof}
The Azuma inequality reads:
\begin{align}
\Pro\(\Big| \sum_{i=1}^{n} X_i \Big| \geq t \) \leq e^{-\frac{2t^2}{4A^2n}},	
\end{align}
see for example \citep{wainwright2019high}.	From here setting the rhs equal to $\delta'$ gives:
\begin{align}
	t \defeq \sqrt{2A^2n\ln\( \frac{1}{\delta'} \)}.
\end{align}
\end{proof}

\begin{lemma}[Change of $\Sigma$-Norm]
\label{lem:MaxEvaluBound}
For a compatible vector $x \in \R^d$ and an spd matrix $\Sigma \in \R^{d\times d}$ with minimum eigenvalue $\lambda_{min}(\Sigma)$ we have
\begin{align}
\|x\|_{\Sigma} & \geq \sqrt{\lambda_{min}(\Sigma)} \| x\|_2 \\
		\|x\|_{\Sigma^{-1}} & \leq \frac{1}{\sqrt{\lambda_{min}(\Sigma)}} \| x\|_2.
\end{align}
\end{lemma}
\begin{proof}
We show one inequality (the other is identical).
Consider the eigendecomposition of $\Sigma$ with orthonormal eigenvectors $v_i$'s and eigenvalues $\lambda_i$'s:
\begin{align}
	\Sigma^{-1} = \sum_{i=1}^{d} \lambda^{-1}_i v_iv_i^\top
\end{align}
We can write:
\begin{align}
		\|x\|^2_{\Sigma^{-1}} & = x^\top\Sigma^{-1} x \\
		& = x^\top\(\sum_{i=1}^{d} \lambda^{-1}_i v_iv_i^\top\) x  \\
		& = \sum_{i=1}^{d} \frac{1}{\lambda_i} \(v_i^\top x\)^2 \\
		& \leq \frac{1}{\lambda_{min}(\Sigma)} \sum_{i=1}^{d} \(v_i^\top x\)^2 \\
		& = \frac{1}{\lambda_{min}(\Sigma)} \| x\|_2^2.
\end{align}
\end{proof}

\begin{lemma}[Linear Bandit Exploration Bonus]
\label{lem:LinearBanditBonus}
For an spd matrix $\Sigma$, the equality below holds whenever the operations make sense:
\begin{align}
	\max_{\phi,\| \eta \|_{\Sigma} \leq \sqrt{\sigma}}   \phi^\top\eta  = \sqrt{\sigma}\|\phi\|_{\Sigma^{-1}}
\end{align}
\end{lemma}

\begin{proof}
	Choose $\eta  = \Sigma^{-1}\phi\frac{\sqrt{\sigma}}{\| \phi \|_{\Sigma^{-1}}}$, which satisfies the constraint
	\begin{align}
		\| \Sigma^{-1}\phi\frac{\sqrt{\sigma}}{\| \phi \|_{\Sigma^{-1}}} \|_{\Sigma} = \| \phi\frac{\sqrt{\sigma}}{\| \phi \|_{\Sigma^{-1}}} \|_{\Sigma^{-1}} \sqrt{\sigma} = \sqrt{\sigma}
	\end{align}
	and gives an objective value
	\begin{align}
			\max_{\phi,\| \eta \|_{\Sigma} \leq \sqrt{\sigma}}   \phi^\top\eta  \geq \phi \Sigma^{-1}\phi\frac{\sqrt{\sigma}}{\| \phi \|_{\Sigma^{-1}}} = \sqrt{\sigma}\| \phi \|_{\Sigma^{-1}}
	\end{align}
	On the other hand, Cauchy-Schwartz ensures:
		\begin{align}
			\max_{\phi,\| \eta \|_{\Sigma} \leq \sqrt{\sigma}}   \phi^\top\eta  \leq \| \phi \|_{\Sigma^{-1}} \| \eta \|_{\Sigma} = \sqrt{\sigma} \| \phi \|_{\Sigma^{-1}}.
	\end{align}
\end{proof}

\end{document}